\documentclass{article}

\usepackage{microtype}
\usepackage{graphicx}
\usepackage{subcaption}
\usepackage{booktabs} %

\usepackage{hyperref}

\usepackage[preprint]{icml2026_fogen}

\usepackage{amsmath}
\usepackage{amssymb}
\usepackage{mathtools}
\usepackage{amsthm}

\usepackage[capitalize,noabbrev]{cleveref}

\theoremstyle{plain}
\newtheorem{theorem}{Theorem}[section]
\newtheorem{proposition}{Proposition}[section]
\newtheorem{lemma}{Lemma}[section]

\theoremstyle{definition}
\newtheorem{definition}{Definition}[section]
\newtheorem{assumption}{Assumption}[section]
\theoremstyle{remark}
\newtheorem{remark}{Remark}[section]

\newcommand{\R}{\mathbb{R}}
\newcommand{\N}{\mathbb{N}}

\newcommand{\ent}{\operatorname{Ent}}

\newcommand{\law}{\operatorname{law}}
\newcommand{\wass}{\operatorname{W}}

\newcommand{\intOT}{\int_{[0,T]}}

\usepackage{dsfont}
\usepackage{xspace}
\usepackage{mathrsfs}
\usepackage{upgreek}

\newcommand{\eg}{e.g.\xspace}

\newcommand{\Rcal}{\mathcal{R}}

\newcommand{\Rd}{{\R^d}}

\newcommand{\E}{\mathbb{E}}
\newcommand{\Eof}[2][]{\mathbb{E}_{#1} \left[ #2 \right]}

\newcommand{\Pof}[2][]{\mathbb{P}_{#1} \left( #2 \right)}

\newcommand{\prob}{\mathbb{P}}
\newcommand{\normof}[1]{\left\Vert #1 \right\Vert}
\newcommand{\frobnorm}[1]{\left\Vert #1 \right\Vert_{\mathrm{F}}}
\newcommand{\frobnormLigne}[1]{\| #1 \|_{\mathrm{F}}}
\newcommand{\klb}[2]{\mathrm{KL}\left(#1 || #2 \right)}

\newcommand{\setof}[1]{\left\{ #1 \right\} }

\newcommand{\landau}[2][]{\mathcal{O}_{#1} \left( #2 \right)}

\newcommand{\der}{\mathrm{d}}

\newcommand{\intrd}{\int_{\Rd}}

\newcommand{\by}[1]{\quad\text{(#1)}}

\newcommand{\tv}{\mathrm{TV}}

\newcommand{\ecal}{\mathcal{E}}

\newcommand{\trace}{{\mathrm{Tr}}}

\newcommand{\Lrm}{\mathrm{L}}

\newcommand{\Irm}{{\mathrm{I}}}

\DeclareMathOperator*{\argmin}{arg\,min}
\newcommand{\Crm}{{\mathrm{C}}}

\newcommand{\fisher}[2]{\mathscr{J}(#1 || #2)}

\newcommand{\ora}[1]{\overrightarrow{#1}}
\newcommand{\rmd}{\mathrm{d}}
\newcommand{\Bm}{\mathrm{B}}

\newcommand{\EofLigne}[2][]{\mathbb{E}_{#1} [ #2 ]}
\newcommand{\normofLigne}[1]{\Vert #1 \Vert}

\newcommand{\pt}{\ora{p}_t}
\newcommand{\ptn}{\ora{p}_t^{(n)}}

\newcommand{\pto}{\ora{p}_{t\vert 0}}
\newcommand{\LDSM}{{\mathscr{L}_{\mathrm{DSM}}}}
\newcommand{\LDSMn}{{\mathscr{L}_{\mathrm{DSM}}^{(n)}}}
\newcommand{\LESMn}{{\mathscr{L}_{\mathrm{ESM}}^{(n)}}}
\newcommand{\LESM}{{\mathscr{L}_{\mathrm{ESM}}}}
\newcommand{\xt}{{\ora{X}_t}}
\newcommand{\xtn}{{\ora{X}_t^{(n)}}}
\newcommand{\xf}{{\ora{X}}}
\newcommand{\xfs}{{\ora{X}^{(n)}}}
\newcommand{\xb}{{\ola{X}}}
\newcommand{\pf}{{\ora{p}}}
\newcommand{\pb}{{\ola{p}}}

 \def\rmd{\mathrm{d}}
 \def\eqsp{\;}

\newcommand{\elbo}{{\operatorname{ELBO}}}

\newcommand{\Sigmabf}{\mathbf{\Sigma}}
\newcommand{\Nrm}{\mathrm{N}}

\newcommand{\rank}{\mathrm{rank}}
\newcommand{\muhat}{\widehat{\mu}}
\newcommand{\Sigmabfhat}{\widehat{\Sigmabf}}
\newcommand{\Ahat}{\widehat{A}}

\renewcommand{\pt}{p_t}
\renewcommand{\ptn}{\hat{p}_t}

\renewcommand{\pto}{p_{t\vert 0}}

\renewcommand{\xt}{{X_t}}
\renewcommand{\xtn}{{\widehat{X}_t}}
\renewcommand{\xf}{{X}}
\renewcommand{\xfs}{{\widehat{X}}}
\renewcommand{\xb}{{Y}}
\renewcommand{\pf}{{p}}
\renewcommand{\pb}{{q}}

\renewcommand{\LDSMn}{{\widehat{\mathscr{L}}_{\mathrm{DSM}}}}
\renewcommand{\LESMn}{{\widehat{\mathscr{L}}_{\mathrm{ESM}}}}

\usepackage[textsize=tiny]{todonotes}

\icmltitlerunning{Benign Overfitting Does Not Occur in Diffusion Models}

\begin{document}
\noindent{}
\vspace{1em}

\icmltitle{Benign Overfitting Does Not Occur in Diffusion Models}

\begin{center}
    \large
  \begin{tabular}{cccc}
    Tyler Farghly$^{1,}$\footnotemark[1] & Benjamin Dupuis$^{1,}$\footnotemark[1] & Alain Durmus$^{2,}$\footnotemark[2] & Umut Şimşekli$^{1,}$\footnotemark[2]
  \end{tabular}

  \begin{tabular}{ccc}
    $^{1}$INRIA - École Normale Supérieure - PSL Research University -  CNRS, France\\
    $^{2}$École Polytechnique - CMAP - IP Paris, France
  \end{tabular}
\end{center}

\footnotetext[1]{Authors contributed equally to this work. Correspondence to \texttt{\{tyler.farghly, benjamin.dupuis\}@inria.fr}.}
\footnotetext[2]{Authors contributed equally to this work.}

\icmlcorrespondingauthor{Tyler Farghly}{tyler-farghly@inria.fr}
\icmlcorrespondingauthor{Benjamin Dupuis}{benjamin.dupuis@inria.fr}

  \icmlkeywords{Machine Learning, ICML}

\vskip 0.3in

\printAffiliationsAndNotice{}  %

\begin{abstract}
Benign overfitting and double descent have come to shape our understanding of generalization in deep learning, establishing that overfitting is not only compatible with good generalization but can actively benefit it. Diffusion models share much of the machinery of standard deep learning, so it is natural to assume that they also exhibit these properties. In this work, we show that this assumption is largely incorrect. We first establish fundamental impossibility results showing that, unless the sample size grows exponentially with the data dimension, overfitting and good generalization cannot occur simultaneously. Consequently, the population loss follows a classical U-shaped curve in model complexity rather than exhibiting double descent. Analyzing a simplified setting, we identify a key difference between regression and score matching: regression benefits from an alignment between the target and the empirical covariance; score matching admits no such alignment, leaving overfitting irreparably harmful. We further identify implicit regularization stemming from time-smoothness of the score and early stopping during training as mechanisms that prevent such overfitting and verify our findings with high-dimensional image generation experiments. Our results reveal that generalization in diffusion models is governed by mechanisms distinct from those of traditional regression, motivating the development of new theory.
\end{abstract}

\section{Introduction}
\label{sec:introduction}

Diffusion models are a popular class of deep generative models \cite{hoDenoisingDiffusion2020,song2021scorebased,sohl_dickstein_deep_2015}. Their key insight is to reframe generation as the reversal of a noising process, characterized by a score function. Estimation of this score function turns out to be equivalent to a denoising problem that can be solved by training a neural network on a simple regression-like objective \cite{song_generative_2019,vincent_connection_2011,hyvarinen2005score}.
Despite their conceptual simplicity, they are capable of achieving state-of-the-art performance in a wide range of applications \cite{dhariwal_diffusion_2021,esser2024scaling,Rombach_2022_CVPR,saharia_photorealistic_2022,watson_novo_2023,zhang_minimax_optimality_2024}, and appear in many of the most widely deployed generative AI systems \cite{yang_diffusion_survey_2023}.

The widespread success of diffusion models naturally raises questions about the mechanisms underpinning their strengths, motivating a growing body of theoretical work on diffusion models \cite{oko_diffusion_2023,azangulov2024convergencediffusionmodelsmanifold,benton2024nearly,conforti2025kl}. 
Of particular interest is the question of generalization \cite{li_generalization_2024,bonnaire_why_2025-1,farghly2025implicitregularisationdiffusionmodels}: understanding when a score network trained on finite data faithfully represents the underlying distribution, rather than memorizing the training set. This question is further sharpened by practical concerns around data privacy \cite{carlini2023extractingtrainingdatadiffusion}, making a principled understanding of the generalization-memorization tradeoff both scientifically and practically important.

\begin{figure*}[t]
    \centering
    \includegraphics[width=.95\linewidth]{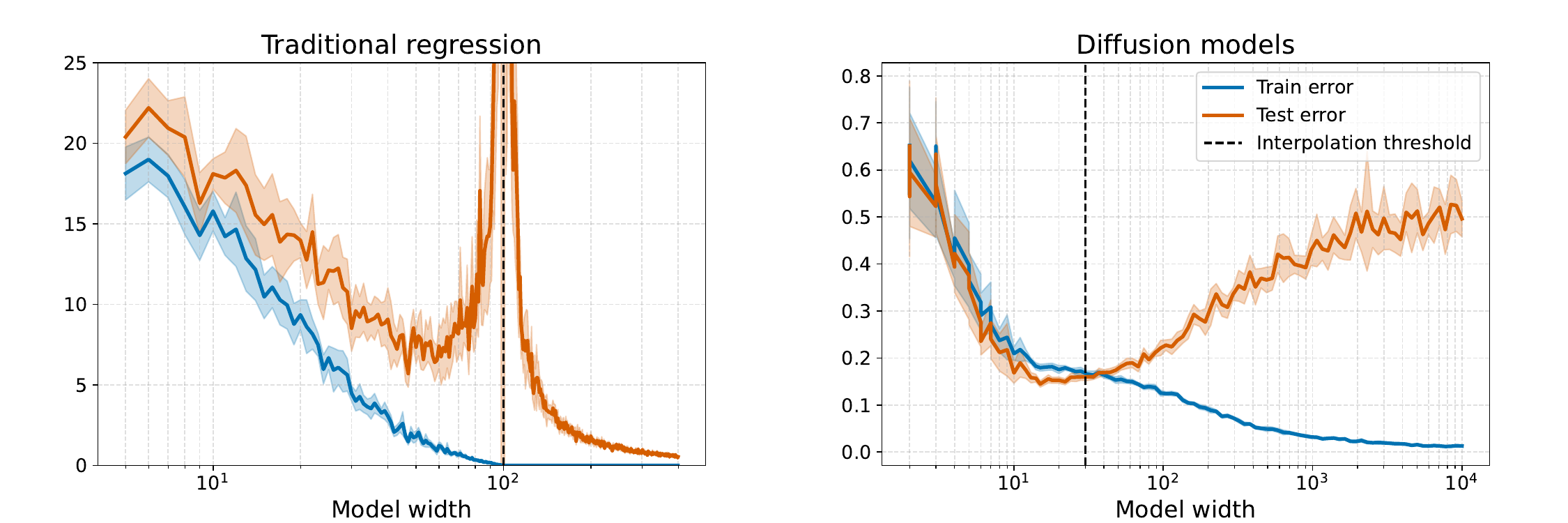}
    \caption{\textit{(Left)} $2$-layer random feature network \textbf{regression}, displaying the classical double descent curve. \textit{(Right)} $2$-layer random feature network \textbf{diffusion}, showing the absence of benign overfitting. Details on the setup and hyperparameters for these experiments are available in \Cref{sec:random-feature-experiment-details}.}
    \label{fig:intro_plot_random_features}
    \vskip -0.4cm
\end{figure*}

Alongside these developments, an extensive body of theory has sought to explain how overparameterized neural networks generalize. These models famously challenge the classical bias--variance tradeoff by generalizing well despite perfectly memorizing noisy training data \cite{zhang_understanding_2017, Shalev-Shwartz_Ben-David_2014}---a phenomenon known as \emph{benign overfitting} \cite{bartlett_benign_2020-1}. This is explained by theoretical analysis showing that generalization is maintained by a form of self-induced regularization that occurs in the overparameterized regime. Furthermore, when taken as a function of model complexity, the test error exhibits a \emph{double descent} \cite{belkin_reconciling_2019}: it first decreases, peaks near the interpolation threshold, and then decreases again as model size grows (see \Cref{fig:intro_plot_random_features}, left).
This behavior was first identified empirically in high-dimensional problems \cite{nakkiran_deep_2019} and then theoretically grounded through the analysis of random feature models \cite{mei_generalization_2022}. Together, these theories form a now widely internalized picture: in highly overparameterized models, fitting the training data perfectly is not only compatible with generalization but, in the right regime, actively beneficial.

One of the benefits of the diffusion model framework is that it reuses much of the previously existing machinery of deep learning, using existing neural network architectures, training with standard optimization algorithms (\eg, ADAM \cite{kingma_adam_2015}), and fitting by minimizing a regression-like objective \cite{vincent_connection_2011,song2021scorebased}. As a result, it is tempting to conclude that the previously mentioned deep learning phenomena are at work, and benign overfitting and double descent carry over more or less intact. \emph{In this work, we show that such an assumption is largely incorrect.}
While several recent works have hinted that diffusion models may be incapable of benign overfitting \cite{farghly2025implicitregularisationdiffusionmodels, dupuis2025algorithmdatadependentgeneralizationbounds, merger2025generalization} and other recent works have explored the relationship between over-training and memorization \cite{Pidstrigach2022-jz, bonnaire_why_2025-1}, the claim that benign overfitting and double descent do not occur in diffusion models is yet to receive rigorous empirical or theoretical investigation.

For diffusion models, we show that in most practical settings, it is fundamentally impossible to have both the train and test losses simultaneously small, invalidating benign overfitting. We further show that, instead of a double descent curve, the excess risk forms a U-shaped curve with increasing complexity, akin to the classical bias--variance tradeoff. This is illustrated in the random feature network experiment in \Cref{fig:intro_plot_random_features} where the excess risk forms a U-shape that tapers off as the width of the network grows large. In \Cref{fig:intro_figure_experiment-unet}, we see how this phenomenon affects high-dimensional image generation problems based on a U-Net architecture \cite{hoDenoisingDiffusion2020,ronnenberger_unet_2015}, showing that as the number of U-Net features is increased, the over-trained diffusion model goes from low quality generation, to generalization, and then towards exactly memorizing the train set. Through a combination of fundamental theoretical results, simplified models, and high-dimensional experiments, we explore this phenomenon, the mechanisms underpinning it and how diffusion models generalize in spite of it.%

Concretely, we present the following contributions:
\begin{itemize}
    \item In \Cref{sec:fundamental-limitations}, we derive fundamental negative results showing that the population and empirical score matching loss cannot be simultaneously small in practical settings, unless the sample size grows exponentially with data dimension, regardless of the score network architecture.
    \item In \Cref{sec:fine-grainded-linear-analysis}, through the analysis of a linear random feature model, we provide a precise understanding of the interplay between model complexity, overfitting and noise scale. By comparing with a classical regression setup, we illustrate the key differences that prevent benign overfitting from occurring in score matching.
    \item In \Cref{sec:preventing-overfitting}, we continue the analysis of \Cref{sec:fine-grainded-linear-analysis}, identifying key components of the score matching framework that can act as implicit regularization and help prevent overfitting in practice. More precisely, we identify the role of time-smoothness of the score and discuss the role of early stopping during training.
    \item We validate our theoretical findings by experiments conducted in high-dimensional image generation settings.
\end{itemize}

\textbf{Notations.} For probability measures $\mu$ and $\nu$ in $\Rd$, if $\mu$ is absolutely continuous with respect to $\nu$, denoted by $\mu \ll \nu$, we denote $\rmd \mu/\rmd \nu$ its density. If $\mu \ll \nu$, the Kullback-Leibler (KL) divergence is $\klb{\mu}{\nu} = \int \log (\der \mu / \der \nu) \der \mu$, the relative Fisher information is $\fisher{\mu}{\nu} = \int \normofLigne{\nabla \log (\der \mu / \der \nu)}^2 \der \mu$, the Fisher information is $\mathscr{J}(\nu) = \int \normofLigne{\nabla \log(\der \nu / \der x)}^2 \der \nu$. For $a,b \in \R$, $a \wedge b = \min(a,b)$, $a \vee b = \max(a,b)$, and $a_+ = a \vee 0$.

\begin{figure*}[t]
    \centering
    \begin{subfigure}[b]{0.31\textwidth}
        \includegraphics[width=\textwidth]{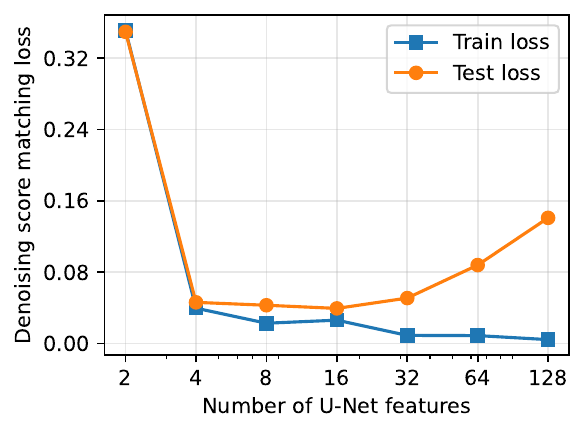}
    \end{subfigure}
    \begin{subfigure}[b]{0.31\textwidth}
        \includegraphics[width=\textwidth]{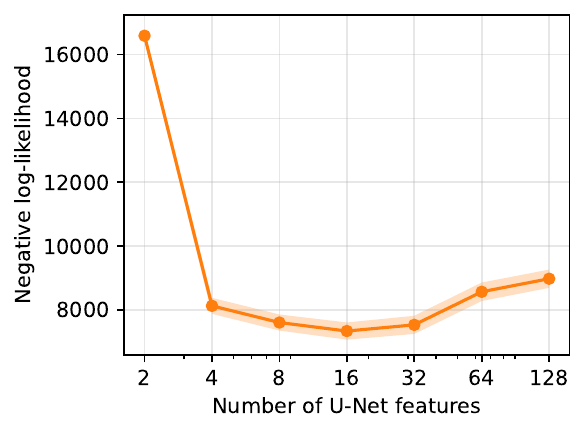}
    \end{subfigure}
    \hfill
    \begin{subfigure}[b]{0.34\textwidth}
        \includegraphics[width=\textwidth]{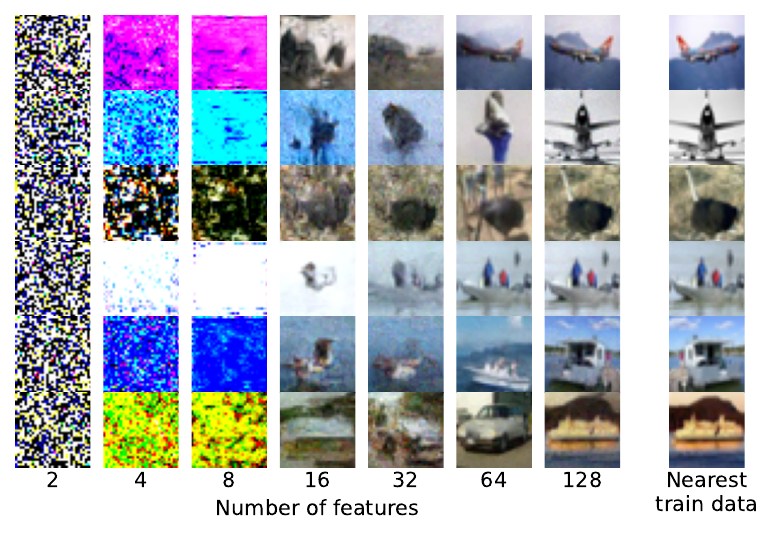}
    \end{subfigure}
    \caption{We train a DDPM U-Net model on a subset of CIFAR10 to convergence. We observe train and test error for different configurations of the model, varying the number of features (NF) in the U-Net. \textit{(Left)} Test and train error at convergence varying NF and thus, varying the number of parameters in the model. \textit{(Center)} The same setting but for the population negative log-likelihood. \textit{(Right)} Generated samples for different NF values alongside the closest train image for $\text{NF}=8$. Details of the experiment are given in Appendix \ref{sec:unet_experiments}.}
    \label{fig:intro_figure_experiment-unet}
    \vskip -0.3cm
\end{figure*}

\section{Background}
\label{sec:background}

By learning to reverse a noising process, diffusion models approximate a distribution $\nu$ on $\R^d$ from a dataset sampled from $\nu$. We briefly review this construction and fix the notation used in the paper.

\textbf{Noising process.}
The noising process is given by an SDE of Ornstein--Uhlenbeck-type in $\Rd$,
\begin{equation}
    \label{eq:forward-process-sde}
    \der \xt = -\kappa \xt \der t + \sqrt{2} \der \Bm_t\eqsp, \qquad \xf_0 \sim \nu\eqsp,
\end{equation}
where $\kappa \geq 0$ and $\Bm_t$ is a standard $d$-dimensional Brownian motion. 
We assume that $\EofLigne[X\sim\nu]{\normofLigne{X}^2} < \infty$.
It is well-known that $\xt|\xf_0 \sim \Nrm(\alpha_t \xf_0, \sigma_t^2 \Irm_d)$, where the mean and variance are given by
\begin{equation}
\label{eq:cond_score_form}
    \alpha_t = e^{- \kappa t} \eqsp, \qquad \sigma^2_t = \kappa^{-1} (1 - \alpha_t^2) \eqsp,
\end{equation}
with the convention $\sigma_t^2 = 2t$ when $\kappa=0$.
Fixing a time horizon $T > 0$ and writing $\xb_t := \xf_{T-t}$, the classical result of \cite{haussmann1986time,anderson_reverse_time_1982} identifies the law of the time-reversed process as the solution of
\begin{equation*}
    \der \xb_t = \kappa \xb_t \der t + 2 \nabla \log \pf_{T-t}(\xb_t) \der t + \sqrt{2} \der \Bm_t\eqsp, \qquad \xb_0 \sim \pf_T \eqsp,
\end{equation*}
where $p_t$ denotes the density of $X_t$.
Thus, if we could simulate $\xb_t$ starting from $\xb_0 \sim \pf_T$, we would obtain samples $Y_T \sim \nu$.
The only unknown quantity is the \emph{score function} $\nabla \log \pf_t$, which must be learned from data. Once we have an approximate score function $\hat{s}:\Rd \times [0,T] \to \Rd$, we generate samples by numerically simulating
\begin{equation*}
    \der \widehat{\xb}_t = \kappa \widehat{\xb}_t \der t + 2 \hat{s}(\widehat{\xb}_t, t) \der t + \sqrt{2} \der \Bm_t\eqsp, \qquad \xb_0 \sim q_0 \eqsp.
\end{equation*}
For $\kappa > 0$, we set $q_0=\Nrm(0, \kappa^{-1} \Irm_d)$ as it approximates $p_T$ for large $T$, and for $\kappa = 0$, we set $q_0 = \Nrm(0, 2 T \Irm_d)$. We denote by $q_t$ the probability density of $\widehat{\xb}_t$ for $t \in [0, T]$.

\textbf{Score matching.}
The score function is approximated using a deep neural network. Ideally, the goal is to minimize the \emph{(population) explicit score matching (ESM) loss},
\begin{align*}
    \LESM({s}, t) := \E \big [ \big \| {s}(t, \xt) - \nabla \log \pt (\xt) \big \|^2 \big ]\eqsp, \qquad \LESM({s}, \varpi) = \int \LESM({s}, t) \, \der \varpi(t)\eqsp,
\end{align*}
where $\varpi$ is some non-negative measure on $[0, T]$.
While the time-weighting $\varpi$ can be chosen freely, a canonical choice is the \emph{evidence lower bound (ELBO) weighting}, $\varpi_\elbo^\epsilon (\der t) = \mathds{1}_{[\epsilon, T]}(t)$, for $\epsilon > 0$.
Under this weighting, the explicit score matching loss upper bounds the marginal KL divergence with $\klb{\pf_\epsilon}{{\pb}_{T-\epsilon}} \leq \LESM (s, \varpi_\elbo^\epsilon) + \klb{\pf_T}{q_0}$.
Since the ESM loss depends on the intractable score $\nabla \log \pt$, a tractable alternative is considered: the \emph{(population) denoising score matching (DSM) loss},
\begin{equation*}
    \LDSM(s, \varpi) := \int \E[\ell_t(s, X_0)] \der \varpi(t)\eqsp, \qquad \ell_t(s, X_0) := \E \big [ \big \| s(t, \xt) - \nabla \log \pto (\xt|X_0) \big \|^2 \big | X_0 \big ]\eqsp,
\end{equation*}
where $p_{t|0}(\cdot|\cdot)$ is the density of $X_t$ given $X_0$. This loss is equal to the ESM loss up to a constant whilst only using the simpler conditional score $\nabla \log \pto$, turning the learning problem into a denoising problem \cite{vincent_connection_2011,hyvarinen2005score}.
It is then estimated using a dataset $S = \{Z_i\}_{i=1}^n \sim \nu^{\otimes n}$, leading to the \emph{empirical DSM loss},
\begin{align}
    \label{eq:emp_dsm}
    \LDSMn(s, \varpi) := 
    \frac1{n} \sum_{i=1}^n \int \ell_t (s, Z_i) \, \der \varpi(t) \eqsp, \qquad n \geq 1\eqsp.
\end{align}
We also define the empirical ESM loss, related to the above by an additive data-dependent constant,
\begin{align}
    \label{eq:emp_esm}
    \LESMn({s}, t) := \E \big [ \big \| {s}(t, \xtn) - \nabla \log \ptn (\xtn) \big \|^2 \big | S \big ]\eqsp, \qquad \LESMn({s}, \varpi) = \int \LESMn({s}, t) \, \der \varpi(t)\eqsp,
\end{align}
where $\der \xtn = -\kappa \xtn \der t + \sqrt{2} \der \Bm_t,$ $\xfs_0 \sim n^{-1} \sum_{i=1}^n \updelta_{Z_i}$ and for $t>0$, $\ptn$ is the density of $\xtn$.

\section{Fundamental limitations of overfitting in diffusion models}
\label{sec:fundamental-limitations}

Benign overfitting occurs when good generalization performance is achieved despite interpolating training data. In this section, we derive impossibility results that show that this cannot happen in diffusion models without an exceedingly large sample size. Unlike \Cref{sec:fine-grainded-linear-analysis,sec:preventing-overfitting}, the results presented here hold regardless of the score network architecture. In this sense, they identify fundamental limitations of the score matching framework itself, rather than of any particular class of models.

\subsection{Warm-up: A Fisher information lower bound}

We begin with the following simple lemma, which is a generic lower bound on the sum of the population and empirical ESM losses. 
The proof can be found in \Cref{proof:fisher-lemma}.

\begin{lemma}
\label{lemma:fisher-lemma}
    Let $t \in (0,T]$, for any measurable score function $s: \Rd \times [0,T] \to \Rd$, we have
    \begin{align}
        \label{eq:fisher-lower-bound}
        \EofLigne{\LESMn (s, t) + \LESM(s, t)} \geq \frac1{2} \Eof{\fisher{\ptn}{\pt}} \eqsp.
    \end{align}
    In particular, if $\mathscr{J}(\nu) < +\infty$, we have that $\inf_s \EofLigne{\LESMn (s, t) + \LESM(s, t)} \to +\infty$ as $t \to 0^+$.
\end{lemma}

\Cref{lemma:fisher-lemma} shows that for both the empirical and population score matching losses to be small, we must have the empirical and population measures $\ptn$ and $\pt$ close also. The measure of closeness here is the relative Fisher information, which under additional concentration assumptions, can be shown to upper bound the KL divergence (see \Cref{sec:fisher-lsi-lower-bound}).
As $\ptn$ is initialized at the empirical distribution $n^{-1} \sum_{i=1}^n \updelta_{Z_i}$ where $(Z_1,\dots, Z_n) \sim \nu^{\otimes n}$, this suggests that benign overfitting only happens when the sample size is large enough. 
We make this argument precise in the next subsection.

\subsection{A quantitative lower bound under time-integration}
\label{sec:quantitative-lower-bound}

We now refine the previous result by obtaining lower bounds on the amount of data required for the empirical and population losses to be small simultaneously. The lower bound is a function of the effective dimension of the data distribution $\nu$ which we capture through the \emph{lower R\'{e}nyi dimension} $d'$, formally defined in \Cref{sec:intrinsic-dimensions-background}. In cases where $\nu$ is supported on a compact manifold, $d'$ is exactly the manifold dimension. We also extend Lemma \ref{lemma:fisher-lemma} by considering the full time-integrated loss as opposed to an individual time-slice, requiring only that the weighting $\varpi$ has positive non-decreasing density on an interval. The bound is generic and holds for any random score function $\hat{s}$ that can be taken as a function of training data. Given a real-valued function $w$, we define the generalized inverse $w^{-}(r) := \inf \{t : w(t) \geq r\}$, with the convention $\inf \emptyset = \infty$.

\begin{theorem}
\label{prop:negative_sm_loss}
Suppose that $\varpi$ is a positive measure supported in $[\epsilon, T]$ with $\epsilon > 0$ and non-decreasing density $w : [\epsilon, T] \to \R_+$, that $\nu$ has finite second moments, and that both $T$ and $d$ are sufficiently large. Suppose that the random score function $\hat{s}$ satisfies
\begin{equation*}
    \E_S[\LESMn(\hat{s}, \varpi)] \leq \varepsilon \eqsp,
\end{equation*}
for some $\varepsilon \leq w(T) / 32$. Then, if any of the following conditions hold:
\begin{align}
\label{eq:negative_sm_loss_0}
    \E_S[\LESM(\hat{s}, \varpi)] \leq \varepsilon, \quad \text{ or } \quad \E_S[\LESM(\hat{s}, \varpi_\elbo^\epsilon)] \leq \tfrac{1}{32}, \quad \text{ or } \quad \klb{\pf_\delta}{\pb_{T-\delta}} \leq \tfrac{1}{32},
\end{align}
and $\delta := w^{-}(32\varepsilon)$ is sufficiently small, then it follows that,
\begin{align*}
     n \geq (8 \sqrt{d} \sigma_\delta )^{- \frac{d'}{2}}  \eqsp,
\end{align*}
where $d'$ is the lower Rényi dimension of $\nu$.
\end{theorem}

The assumption on the measure $\varpi$, encompasses several standard choices of time-weighting within the \emph{continuous training} framework \cite{song2021scorebased}. In the following remark, we consider several of these, obtaining closed-form bounds.

\begin{remark}
Under the ELBO weighting $\varpi = \varpi^\epsilon_{\mathrm{ELBO}}$, the sample size lower bound becomes $n \geq (8\sigma_\epsilon \sqrt{d})^{-d'/2}$ for $\epsilon$ sufficiently small. Under the variance weighting $w(t) = \sigma_t^2$, commonly used for training \cite{hoDenoisingDiffusion2020}, the analogous bound is $n \geq (2(32\varepsilon \vee \sigma_\epsilon)\sqrt{d})^{-d'/2}$ for $\varepsilon, \epsilon$ sufficiently small. Most strikingly, when $\epsilon = 0$ and $\varepsilon \leq \tfrac{1}{32}\limsup_{t \to 0^+} w(t)$, there is no sample size at which benign overfitting can occur: $\E[\LESMn(\hat{s}, \varpi)] \leq \varepsilon$ necessarily implies $\E[\LESM(\hat{s}, \varpi)] > \varepsilon$, $\E[\LESM(\hat{s}, \varpi^0_{\mathrm{ELBO}})] > {1}/{32}$, and $\klb{p_\delta}{q_{T-\delta}} > {1}/{32}$.
\end{remark}

We highlight that the lower bound is increasing exponentially in the intrinsic dimension and depends on $\epsilon$, supporting both the idea that manifold distributed data is beneficial to diffusion models and that early stopping in the reverse process can benefit generalization \cite{azangulov2024convergencediffusionmodelsmanifold,farghly2026diffusion,debortoli2022convergence, farghly2025implicitregularisationdiffusionmodels}. Since many popular implementations of diffusion models take $\epsilon \in \setof{10^{-3}, 10^{-5}}$, the expectation that the amount of data scales with $\epsilon^{-d'/2}$ is unreasonably strong. Thus, in most practical settings, diffusion models will not exhibit the property of benign overfitting.

The proof, presented in \Cref{sec:proof-negative-sm-result}, uses a similar approach as \Cref{lemma:fisher-lemma}, showing that for empirical and population losses to be small, $\ptn$ and $\pt$ must be close. Instead of the Fisher information, we derive a connection to the total variation via an argument based on Girsanov's theorem.

\section{Fine-grained analysis on linear RFNNs}
\label{sec:fine-grainded-linear-analysis}

\begin{figure*}[t]
    \centering
    \begin{minipage}{0.32\textwidth}
        \centering
        \includegraphics[width=\linewidth]{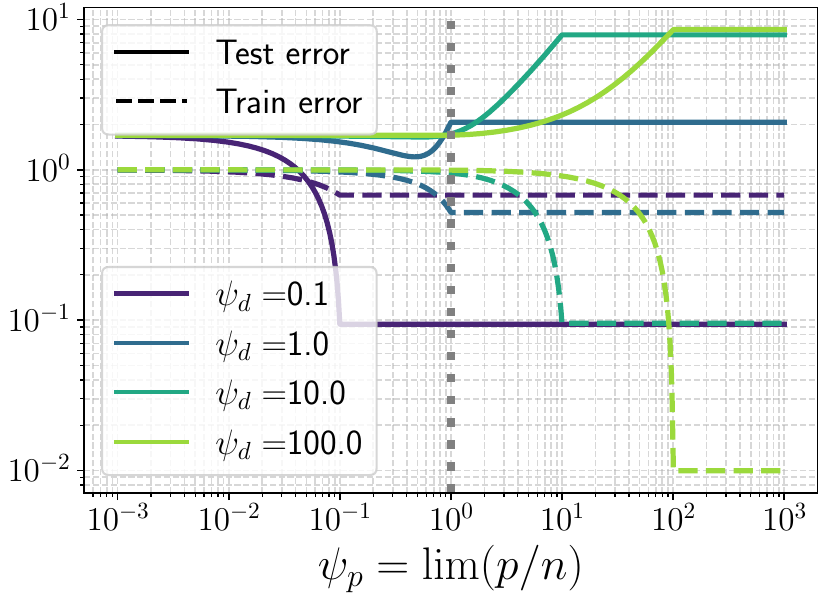}
    \end{minipage}
    \hfill
    \begin{minipage}{0.32\textwidth}
        \centering
        \includegraphics[width=\linewidth]{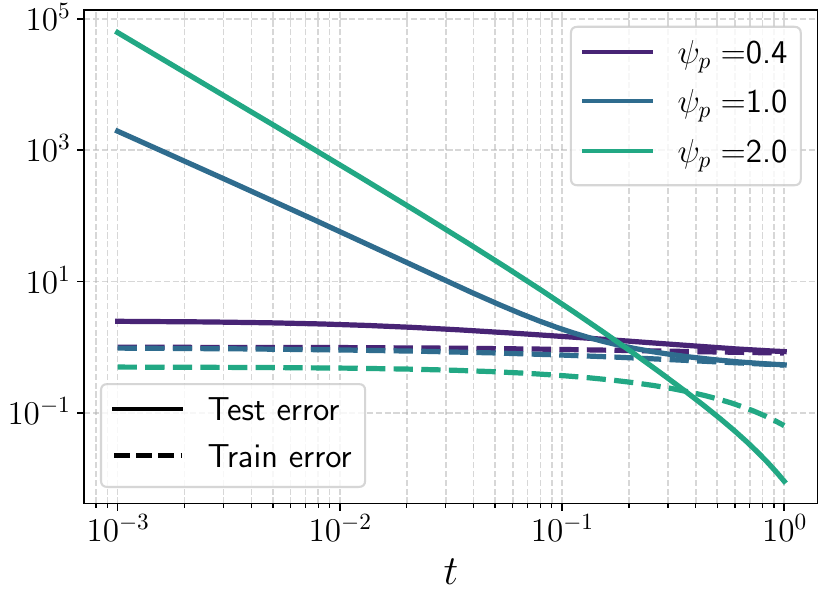}
    \end{minipage}
    \hfill
    \begin{minipage}{0.32\textwidth}
        \centering
        \includegraphics[width=\linewidth]{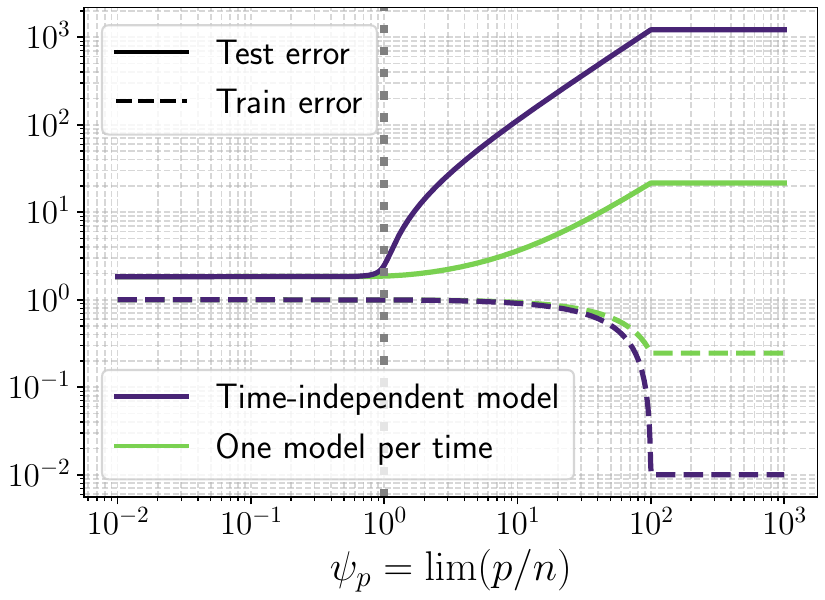}
    \end{minipage}
    \caption{Linear random features under \Cref{ass:linear_setup_assumption}. \textit{(Left)} time is fixed at $t = 10^{-1}$ and $\psi_p$ varies. \textit{(Center)} $\psi_d$ is fixed at $10^2$ and $t$ varies. \textit{(Right)} time integrated losses (as in \Cref{sec:time-smoothness}) with $\pi$ the uniform distribution on $[0,T]$. \texttt{Train error} denotes the empirical DSM loss with the weighting used in training in practice. \texttt{Test error} is the population ESM loss with ELBO weighting.}
    \label{fig:fixed-time-linear-psi_p}
    \vskip -0.4cm
\end{figure*}

Section \ref{sec:fundamental-limitations} establishes at a general level that benign overfitting cannot occur in diffusion models in most practical settings. However, it does not reveal how overfitting manifests, nor how it depends on the interplay between model size, data dimension, and noise scale. To get a finer understanding, we now analyze a concrete and tractable model class: two-layer linear random feature networks \cite{rahimi_random_features_2007}, similar to the linear activation case of the models considered in \cite{george_asymptotic_2026,george_denoising_2025,bonnaire_why_2025-1}. Despite the simplicity of the linear setting, it has been central to the theoretical analysis of benign overfitting and double descent \cite{belkin_reconciling_2019, Hastie2022-zv}, suggesting that it is rich enough to capture overfitting in diffusion models.

\textbf{Setup.} We study $2$-layer linear score networks of the form $s_{W,A} : x \mapsto A W x$, where $W \in \R^{p \times d}$ is fixed and $A \in \R^{d \times p}$ is a trainable parameter. We fix a time $t \in (0,T]$ and consider the minimization of the empirical risk defined by $A \mapsto \sigma_t^2 \LDSMn(s_{W,A}, t)$.
We denote by $\Sigmabf$ the covariance matrix of the data distribution $\nu$ and define the empirical covariance matrix $\Sigmabfhat := n^{-1} \sum_{i=1}^n Z_i Z_i^T$ along with $\Sigmabf_t := \alpha_t^2 \Sigmabf + \sigma_t^2 \Irm_d$ and $\widehat{\Sigmabf}_t := \alpha_t^2 \widehat{\Sigmabf} + \sigma_t^2 \Irm_d$.

In the following lemma, we characterize the DSM and ESM loss associated with the empirical risk minimization problem. The proof can be found in \Cref{sec:proofs-fine-grainded-linear-analysis}.

\begin{lemma}
    \label{lemma:fixed-time-non-asymptotic-lemma}
    Assume that $W$ has full rank and take $\widehat{A} \in \argmin \LDSMn(s_{W, A}, t)$.
    Let $\Omega_t := \Sigmabfhat_t^{-1}$ if $p \geq d$ and $\Omega_t := W^T (W \widehat{\Sigmabf}_t  W^T)^{-1} W$ if $p < d$. Then, $\Ahat W = -\Omega_t$ and,
    \begin{align*}
        \sigma_t^2 \LDSMn(s_{W,\Ahat}, t) = d - \sigma_t^2 \trace ( \Omega_t ) \eqsp, \qquad  \LESM(s_{W, \widehat{A}}, t) = \frobnormLigne{( \Omega_t - \Sigmabf_t^{-1} ) \Sigmabf_t^{1 / 2}}^2 \eqsp.
    \end{align*}
\end{lemma}

For $p\geq d$, the ESM and DSM loss at $\Ahat$ are independent of $p$. By the proof of \Cref{lemma:fixed-time-non-asymptotic-lemma}, this is because the empirical risk minimization (ERM) problem becomes equivalent to the ERM associated with a $d$-dimensional linear model (independent of $p$), similar to \cite{merger2025generalization}.
See \Cref{sec:proofs-fine-grainded-linear-analysis} for details.
We specify our setup further in the next assumption.

\begin{assumption}
    \label{ass:linear_setup_assumption}
    $W\in\R^{p \times d}$ has $\Nrm(0,1)$ i.i.d. entries, $A\in\R^{d \times p}$ and $\nu = \Nrm(0, \beta\Irm_d)$, with $\beta > 0$. 
\end{assumption}

The next proposition is a precise characterization of overfitting in a proportional asymptotic regime.

\begin{proposition}[Overfitting in linear networks]
\label{prop:fixed-time-asymptotic-linear-results}
Under \Cref{ass:linear_setup_assumption}, we consider the limit $d,n,p \to \infty$ such that $p / n \to \psi_p$ and $d / n \to \psi_d$.
When $\psi_d, \psi_p > 1$, we have, almost-surely,
\begin{align}
    \label{eq:linear-asymptotics-fixed-time-population}
   \frac1{d} \LESM(s_{W,\Ahat}, t)  \to \frac1{\alpha_t^2 \beta + \sigma_t^2} \bigg ( \frac{\psi_p \wedge \psi_d - 1}{\psi_d} \frac{\alpha_t^4 \beta^2}{\sigma_t^4} + \left( 1 - \frac{\psi_p}{\psi_d} \right)_+ \bigg ) + \landau{\frac1{\psi_d}} \eqsp,
\end{align}
where $\widehat{A} \in \argmin \LDSMn(s_{W, A}, t)$.
For the empirical risk, we have
\begin{align}
    \label{eq:linear-asymptotics-fixed-time-empirical}
    \frac{\sigma_t^2}{d} \LDSMn (s_{W, \widehat{A}}, t) \to  \frac1{\psi_d} + \left( 1 - \frac{\psi_p}{\psi_d} \right)_+ + \landau{\frac1{\psi_d (\psi_p \wedge \psi_d)}} \eqsp.
\end{align}
\end{proposition}

Proposition \ref{prop:fixed-time-asymptotic-linear-results} demonstrates how the population loss is a U-shaped function of model complexity, captured by the relative model width, $\psi_p$. When $\psi_p > \psi_d$, we observe that the empirical risk of \eqref{eq:linear-asymptotics-fixed-time-empirical} is of order $\landau{\psi_d^{-1}}$ and can, therefore, become arbitrarily small and interpolate training data when the data dimension is large. In the same regime, we observe that the population score matching loss is of order $\sigma_t^{-4}$ and therefore explodes as $t\to 0^+$. This can also be observed clearly in \Cref{fig:fixed-time-linear-psi_p}, displaying the exact asymptotics of the risks, obtained in the proofs in \Cref{sec:proofs-fine-grainded-linear-analysis}.
It demonstrates that interpolation and generalization are mutually incompatible. We also study the dependence on $t$ in \Cref{fig:fixed-time-linear-psi_p}, showing that the rate of explosion of the excess risk as $t \to 0^+$ depends on the model complexity.

\textbf{Comparison with regression.} To understand how these results differ from traditional regression, we consider a multi-output regression problem with the same data and architecture. Given a target $B_\star \in \R^{d \times d}$ and noiseless labels $Y_i := B_\star Z_i$, we minimize the empirical risk $A \mapsto n^{-1} \sum_{i=1}^n \normof{s_{W, A}(Z_i) - Y_i}^2 + \lambda\frobnorm{AW}^2 $. 
We focus on the overparameterized regime $p > d > n$ and $\lambda = 0$, where minimizers $A$ satisfy
$(AW - B_\star) \widehat{\Sigmabf} = 0$.
Since  $\Sigmabfhat$ has rank $n < d$ a.s.,
this gives a $d(d-n)$-dimensional affine subspace of solutions for $AW$, contrasting sharply with \Cref{lemma:fixed-time-non-asymptotic-lemma}, where score matching produces a unique value of $A W = -\Sigmabfhat_t^{-1}$ even without regularization.

The excess risk for the regression problem is given directly by $\frobnormLigne{(AW - B_\star) \Sigmabf^{1/2}}^2$. The expression has a similar form as the score matching ESM loss in Lemma~\ref{lemma:fixed-time-non-asymptotic-lemma} once replacing $B_\star$ with $-\Sigmabf_t^{-1}$, and $\Sigmabf$ with $\Sigmabf_t$ in the weighting.
To resolve the under-determination, we consider the ridgeless limit $\lambda \to 0^+$, recovering the minimum-norm interpolator. This projects $AW$ on the orthogonal of $\mathrm{ker}(\Sigmabfhat)$, leading to the excess risk $\frobnormLigne{B_\star \, \Pi_{\mathrm{ker}(\widehat{\Sigmabf})} \Sigmabf^{1/2}}^2$, where $\Pi_{\mathrm{ker}(\widehat{\Sigmabf})}$ is the column-wise projection onto the kernel of $\widehat{\Sigmabf}$. Thus, generalization can occur in the overparameterized regime due to two factors: (i) the minimum-norm bias projects $AW$ on the orthogonal of $\mathrm{ker}(\Sigmabfhat)$; (ii) the residual is small when $B_\star$ \emph{aligns} with $\Sigmabfhat^{1/2}$ within the kernel of $\widehat{\Sigmabf}$. This is made precise in the benign overfitting literature that identifies a form of \emph{self-induced regularization} resulting from alignment \cite{bartlett_benign_2020-1, Hastie2022-zv}.

The picture is fundamentally different for score matching. 
First, since the empirical problem exactly determines the solution $AW = - \Sigmabfhat_t^{-1}$, there can be no minimum-norm bias that projects $AW$ orthogonally to $\mathrm{ker}(\Sigmabfhat)$. Moreover, whereas $B_\star$ could reasonably align with $\Sigmabf^{1/2}$ on the null space in such a way that makes their product small, it is less likely that $-\Sigmabf_t^{-1}$ aligns in the same way with $\Sigmabf_t^{1/2}$. Indeed, if $\Sigmabf$ is not full rank, we have $\frobnormLigne{\Sigmabf_t^{-1} \Pi_{\mathrm{ker}(\widehat{\Sigmabf})} \Sigmabf_t^{1/2}}^2 \geq (d-n)\sigma_t^{-2}$ which explodes as $t \to 0$.
This is made precise in the proof of \Cref{prop:fixed-time-asymptotic-linear-results}, leading to the first term of \eqref{eq:linear-asymptotics-fixed-time-population}, corresponding to the contribution of the null space of $\Sigmabfhat$ to the ESM loss.

\section{Preventing overfitting in diffusion models}
\label{sec:preventing-overfitting}

The absence of benign overfitting in diffusion models suggests that regularization is required to achieve good generalization in the overparameterized regime. In this section, we show that the score matching framework contains forms of \emph{implicit regularization} that help prevent overfitting.

\subsection{Implicit regularization due to time smoothness}
\label{sec:time-smoothness}

In most practical applications of diffusion models, a single neural network is trained across multiple noise levels. Whilst primarily done for efficiency, the resulting time-smoothness of the score function has been shown empirically to benefit generalization in diffusion models \cite{NEURIPS2025_9a3b1949}. Here, we study the impact of time-smoothness of the score estimator in the setting of \Cref{sec:fine-grainded-linear-analysis} by considering the extreme case where it is time-independent.

\textbf{Setup.}
We parameterize the score network as in \Cref{sec:fine-grainded-linear-analysis} with $p > d$ and we consider a time weighting $\varpi$ which is a probability distribution on $[0,T]$ such that we can define the probability measure $\pi(\der t) \propto \sigma_t^{-2} \varpi(\der t)$.
For any distribution $\omega \in \setof{\pi, \varpi}$ on $[0,T]$, we write
\begin{align*}
    \alpha_\omega^2 := \int \alpha_t^2 \der \omega(t)\eqsp, \quad \sigma_\omega^2 := \int \sigma_t^2 \der \omega(t)\eqsp, \quad \Sigmabf_\omega := \int \Sigmabf_t \der \omega(t)\eqsp, \quad \Sigmabfhat_\omega := \int \Sigmabfhat_t \der \omega(t)\eqsp.
\end{align*}
In the next lemma, we characterize the empirical risk minimization in this setting.

\begin{lemma}
    \label{lemma:time-smoothness-non-asymptitc-lemma}
    Assume that $p > d$ and $W$ is full-rank. Then, any $\Ahat \in \argmin_A \LDSMn(s_{A,W}, \varpi)$ satisfies $\Ahat W = -\Sigmabfhat^{-1}_\varpi$ and we have
    \begin{align*}
        \LESM(s_{\Ahat,W}, \pi) = \frobnormLigne{( \Sigmabfhat_\varpi^{-1} - \Sigmabf_\pi^{-1} ) \Sigmabf_\pi^{1 / 2} }^2 + \Crm_\pi \eqsp, \quad \LDSMn(s_{\Ahat,W}, \varpi) = d \sigma_\pi^{-2} - \trace (\Sigmabfhat_\varpi^{-1}) \eqsp,
    \end{align*}
    where $\Crm_\pi \geq 0$ is given in the proof in \Cref{eq:time-integrated-ESM_constant}. $\Crm_\pi$ is $0$ only when $\pi$ is a Dirac distribution.
\end{lemma}

\begin{remark}[Equivalence with ridge]
    \label{rk:ridge-equivalence}
    By \Cref{lemma:time-smoothness-non-asymptitc-lemma}, we have $\Ahat W = -(\alpha_\varpi^2 \Sigmabfhat + \sigma_\varpi^2 \Irm_d)^{-1} = - \Sigmabfhat_\varpi^{-1}$. As $\varpi$ is a probability distribution, there exists $t_\varpi > 0$ such that $\alpha_{t_\varpi}^2 = \alpha_\varpi^2$. Thus, for any $t < t_\varpi$, we have
\begin{align*}
   \Sigmabfhat_\varpi = \eta_{t,\varpi}^{-1} ( \Sigmabfhat_t + 2\lambda_{t,\varpi} \Irm_d) \eqsp, \quad 2\lambda_{t,\varpi} := \eta_{t,\varpi} (\sigma_\varpi^2 - \sigma_t^2) > 0 \eqsp, \quad \eta_{t,\varpi} := \alpha_t^2 / \alpha_\varpi^2 \eqsp,
\end{align*}
and so, we see that $\eta_{t, \varpi}^{-1} \Ahat \in \argmin_A (\LDSMn(s_{A,W}, t) + \lambda_{t,\varpi} \normofLigne{AW}^2)$. Therefore, in this case, considering the time-integrated loss is formally similar to adding ridge regularization for small times $t < t_\varpi$.
\end{remark}

Next, we make this observation precise by characterizing the asymptotic score matching losses.

\begin{proposition}
    \label{prop:time-smoothness-asymptotics}
    Under \Cref{ass:linear_setup_assumption}, we consider the limit $d,n,p \to \infty$ such that $p / n \to \psi_p$ and $d / n \to \psi_d$. Let $\Ahat \in \argmin_A \LDSMn(s_{W,A}, \varpi)$. When $\psi_p > \psi_d > 1$, we have, almost-surely,
    \begin{align*}
         \frac{\LDSMn(s_{W,\Ahat}, \varpi)}{d} \to
         \frac1{\sigma_\pi^2} - \frac1{\sigma_\varpi^2}  + \mathcal{O}(\psi_d^{-1}), \quad \frac{\LESM(s_{W,\Ahat}, \pi )}{d} \to
         \Lrm_\pi + \mathcal{O}(\psi_d^{-1}) ,
    \end{align*}
    where $ {\Lrm}_\pi \geq 0$ is a constant whose expression is given in \Cref{eq:Lrm_pi_constant}, and $\sigma_\pi^{-2} - \sigma_\varpi^{-2} \geq 0$.
\end{proposition}

\Cref{prop:time-smoothness-asymptotics} shows that time-smoothness effectively prevents the explosion of the population loss observed in \Cref{prop:fixed-time-asymptotic-linear-results}. It also shows that the limiting empirical risk cannot be arbitrarily small as soon as $\varpi$ is not a Dirac distribution (by similar calculations, we reach the same conclusion for $\LESMn(s_{W,\Ahat}, \varpi)$).
This contrasts with the fixed noise scale case of \Cref{prop:fixed-time-asymptotic-linear-results}, even when $\psi_d$ is large, suggesting that the time-smoothness prevents overfitting. Moreover, while in \Cref{prop:fixed-time-asymptotic-linear-results} the asymptotic ESM loss could explode as $t \to 0^+$, in the present setting the ESM loss stays bounded, showing that the time-smoothness acts as a form of implicit regularization, as hinted in \Cref{rk:ridge-equivalence}.

In the proof of \Cref{prop:time-smoothness-asymptotics} (see \Cref{sec:proofs-time-smoothness}), we derive exact asymptotics for the population and empirical losses, which we report in the right figure in \Cref{fig:fixed-time-linear-psi_p} (green curves). Using the asymptotics of the previous section, we compare it against an overfit linear score network that has independent weights at each time, corresponding to the completely non-time-smooth case (purple curves).
We observe that time-smoothness increases the empirical risk while decreasing the population loss, showing that it is indeed acting as an implicit regularization mechanism.

We verify this in a practical image generation setting similar to \Cref{fig:intro_figure_experiment-unet}, but considering a time-independent DDPM U-Net model trained across a range $[t, t+r]$ for varying $r$. We evaluate train and test error at $t$ to see how changing $r$ affects the error. We present the results in \Cref{fig:no_time_exp} where we find that increasing $r$ increases the train error at $t$ and therefore is indeed preventing overfitting. 
Interestingly, the test error first decreases and then increases for large $r$, suggesting a tradeoff on $r$.

\subsection{Implicit regularization via early stopping}
\label{sec:early-stopping}

Next, we study the impact of early stopping during the training of the score network.
To this end, we fix a time $t \in [0,T]$ and study the dynamics of the parameter $A$ under gradient flow,
\begin{align}
    \label{eq:gradient-flow-dynamics}
    \frac{\der}{\der \tau}A(\tau) = -p^{-1}\nabla_A (\sigma_t^2 \LDSMn(s_{W,A(\tau)}, t)) \eqsp, \quad \tau \geq 0 \eqsp,
\end{align}
where $W \in \R^{p \times d}$ is as in \Cref{sec:fine-grainded-linear-analysis}, and the $p^{-1}$ factor is ensuring a finite limit as $p \to \infty$ \cite{bonnaire_why_2025-1}.
Without loss of generality, we assume that $A(0) = 0$. 
We show in \Cref{sec:proof-early-stopping} that in the limit $p \to \infty$ with $(n,d)$ fixed, the matrix $A(\tau) W$ converges to a finite limit following the gradient flow of a linear model (as shown in \Cref{sec:proof-early-stopping}). 
Moreover, when $\tau \to \infty$, $A(\tau) W \to - \Sigmabfhat^{-1}_t$, which leads to the same score matching bounds as in \Cref{prop:fixed-time-asymptotic-linear-results} for $p > d$. In the next proposition, we show that for small enough values of $\tau$, we avoid the explosion of the test loss in \Cref{prop:fixed-time-asymptotic-linear-results}.

\begin{proposition}
    \label{prop:early-stopping-proposition}
    Suppose that \Cref{ass:linear_setup_assumption} holds and define $\ecal_t(\tau,n,d,p) := \LESM(s_{W,A(\tau)}, t)$. Then, $\ecal_t(\tau,n,d,p)$ has an almost-sure limit $ \overline{\ecal}_t(\tau, n, d) := \lim_{p\to \infty} \LESM(s_{W,A(\tau)}, t)$.
     Moreover, in the limit as $n,d \to \infty$ with $ \lim(d / n) = \psi_d > 1$, we have, almost-surely
    \begin{align*}
        \frac1{d}\overline{\ecal}_t(\tau, n, d) \longrightarrow \left( 1 - \frac1{\psi_d} \right) \frac{\left((\alpha_t^2 \beta + \sigma_t^2) \sigma_t^{-2} (1 - e^{-2\tau \sigma_t^4}) - 1 \right)^2 }{\alpha_t^2 \beta + \sigma_t^2}  + \landau{\frac1{\psi_d}} \eqsp,
    \end{align*}
    In particular, if $\tau = \landau{\sigma_t^{-2}}$, then the limit above is bounded as $t \to 0^+$.
\end{proposition}

This result shows that early stopping prevents the explosion of the test loss as $t \to 0$ given that the early stopping $\tau$ is of order $\sigma_t^{-2}$. This suggests that, despite not being benign, overfitting and memorization appear later in training, allowing early stopping to effectively prevent it, which aligns with recent results \cite{bonnaire_why_2025-1,merger2025generalization}.
The advantages of early stopping can also be observed in high-dimensional experiments: in \Cref{fig:train_unet}, we report the evolution of the test and train errors for a U-Net across iterations, observing that there is a unique optimal stopping time where the ESM loss is controlled as the number of parameters grows.

\begin{figure*}
    \centering
    \begin{minipage}{0.48\textwidth}
        \centering
        \includegraphics[width=.92\linewidth]{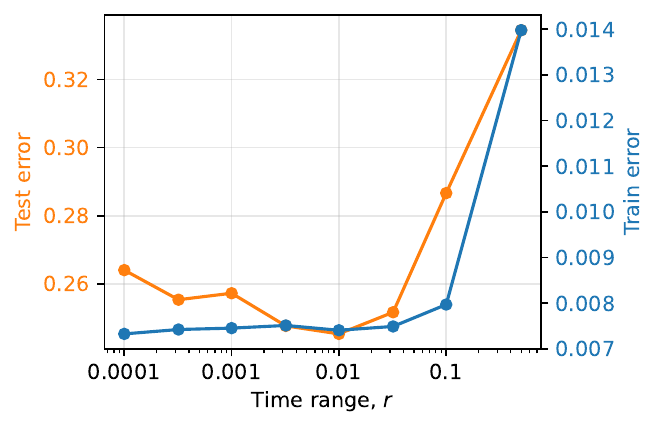}
        \caption{Time-independent DDPM model evaluated at $t=0.1$. We train it on the range $[t, t+r]$ and observe how increasing $r$ can benefit generalization at time $t$. Details in Appendix \ref{sec:unet_experiments}.}
        \label{fig:no_time_exp}
    \end{minipage}
    \hfill
    \begin{minipage}{0.48\textwidth}
        \centering
        \includegraphics[width=.9\linewidth]{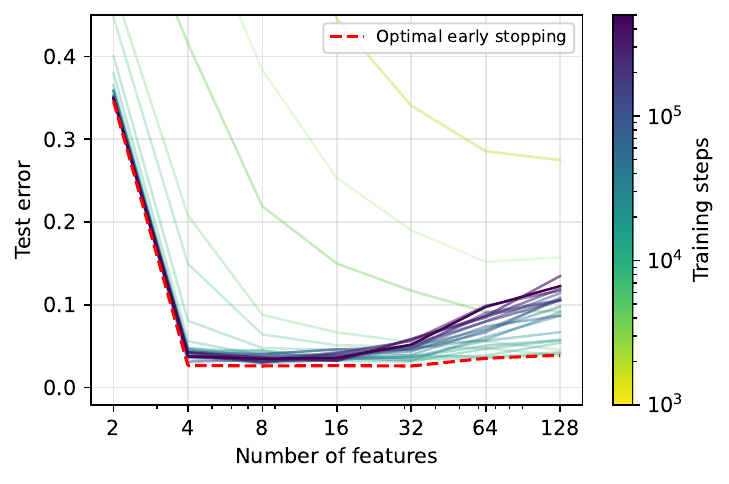}
        \caption{The same setting as Figure \ref{fig:intro_figure_experiment-unet} but we observe how the test error curve changes with training steps. We find that early stopping produces a test error that does not grow with NF.}
        \label{fig:train_unet}
    \end{minipage}
    \label{fig:early_stopping_figure}
    \vskip -0.4cm
\end{figure*}

\section{Related works \& Conclusion}
\label{sec:related_works}

\textbf{Random features analysis \& linear models.}
While we focus on quantifying overfitting, we also derive exact asymptotics of the losses in the proof of \Cref{prop:fixed-time-asymptotic-linear-results}, complementing prior analyses \cite{george_asymptotic_2026,george_denoising_2025,merger2025generalization}.
As discussed in \Cref{sec:fine-grainded-linear-analysis}, the two-layer linear RFNNs analysis recovers linear diffusion models as a particular case, extending the regularized linear model analysis of \cite{merger2025generalization}. In particular, we recover the conclusion that overfitting is driven by the null space. 
Recent work also studied the learning dynamics of diffusion models \cite{Biroli2024DynamicalRO,merger2025generalization}, including RFNN settings \cite{bonnaire_why_2025-1,li_generalization_2024}.  \Cref{sec:early-stopping} supplements these approaches by quantifying the role of early stopping in mitigating overfitting in linear networks.
Overall, we provide a fine-grained characterization of overfitting in linear RFNNs and, crucially, incorporate the \textbf{time-smoothness} of the models with time-integrated losses (\Cref{sec:time-smoothness}), which is beyond the scope of prior work.

\textbf{Diffusion model generalization.} In addition to the aforementioned works, there has been a lot of recent interest in the generalization properties of diffusion models \cite{kadkhodaie2024generalization,li_generalization_2024,chakraborty2026generalization}.
In particular, minimax statistical rates have been obtained \cite{oko_diffusion_2023,zhang_minimax_optimality_2024,azangulov2024convergencediffusionmodelsmanifold},
the relationship with data geometry has been explored \cite{farghly2026diffusion, Pidstrigach2022-jz, He2026-lp}
and, more recently, algorithm-dependent generalization bounds were proposed for diffusion models \cite{dupuis2025algorithmdatadependentgeneralizationbounds,farghly2025implicitregularisationdiffusionmodels}. Our work offers a new but complementary perspective, by suggesting that the mechanisms causing generalization in diffusion models are fundamentally different from those in classical settings.

\textbf{Statistical lower bounds.} Recent work analyzed memorization through the development of lower bounds \cite{buchanan2025edgememorizationdiffusionmodels}.
In particular, \Cref{lemma:fisher-lemma} can be compared with \cite{ye2026provable}, who showed that $\fisher{\ptn}{\pt} = \LDSMn (\nabla \log \pt, t) - \LDSMn (\nabla \log \ptn, t)$.
This is of a different nature than the left-hand side of \Cref{eq:fisher-lower-bound}, which is the sum of the empirical and population ESM losses for a generic score.
That being said, the lower bounds on $\fisher{\ptn}{\pt}$ derived in \cite{ye2026provable} for mixture data could easily be combined with \Cref{lemma:fisher-lemma}. Our results in \Cref{sec:quantitative-lower-bound} differ in that they consider a very generic setting and exploit the intrinsic dimension of the data in statistical guarantees. 
Note that \eqref{eq:fisher-lower-bound} does not require any restrictive assumptions on $\nu$, such as the log-Sobolev inequality used in \cite{koehler2023statistical}.

\paragraph{Conclusion.} In this work, we showed that the phenomenon of benign overfitting, known to occur in many deep learning settings, does not occur in diffusion models. We first derive a fundamental limitation of score matching, showing that the empirical and population losses cannot be small without an exceedingly large training set. Then, through a combination of high-dimensional experiments and a fine-grained analysis in a simplified setting, we showed that the double descent phenomenon of deep learning does not occur and that the test error follows a U-shaped curve with respect to model complexity. Finally, by incorporating the time-smoothness of the score and early-stopped gradient dynamics into our experiments and linear network analysis, we demonstrated that these key components of score estimation can help prevent overfitting.

\section*{Acknowledgements}

U.S. is partially supported by the French government under the management of Agence Nationale de la Recherche as part of the ``Investissements d'avenir'' program, reference ANR-19-P3IA-0001 (PRAIRIE 3IA Institute). T.F., B.D., and U.S. are  supported by the European Research Council Starting Grant DYNASTY – 101039676.
A.D. is supported by the France 2030 program with the reference ANR-25-PEIA-0001 (THEOREM project). A.D. is funded by the European Union (ERC-2022-SYG-OCEAN-101071601). Views and opinions expressed are however those of the author(s) only and do not necessarily reflect those of the European Union or the European Research Council Executive Agency. Neither the European Union nor the granting authority can be held responsible for them. A.D. is supported by Hi! Paris and Agence Nationale de la Recherche (Grant 11-LABX-0047).
This work received government funding administered by the National Research Agency (ANR) under the France 2030 program "Hi! PARIS", grant number ANR-23-IACL-0005.
The authors would also like to thank George Deligiannidis, Giovanni Conforti and Francis Bach for inspiring discussions. This paper was partially motivated by a mid-seminar argument with Francis and the SIERRA team initiated by a statement similar to the title of this work.

\bibliography{thesis.bib}
\bibliographystyle{icml2026_fogen}

\newpage

\appendix
\onecolumn
\section{Additional Technical Background}
\label{sec:additional-technical-background}

\subsection{Intrinsic dimensions}
\label{sec:intrinsic-dimensions-background}

Our main results in \Cref{sec:fundamental-limitations} are expressed in terms of the \emph{Rényi} dimension of a Borel measure $\mu$ on $\Rd$. This notion was first developed by \cite{procaccia1983} and then rigorously treated in \cite{pesin_dimension_1997,pesin_rigorous_1993}. It is a standard notion in the theory of chaotic dynamical systems.
We provide below the definition.

\begin{definition}[Rényi dimension of a measure]
    \label{def:correlation-dimension}
    Let $\mu$ be a Borel probability measure on $\Rd$. The upper Rényi dimension of $\mu$ is defined by
    \begin{align*}
        \overline{\gamma} (\mu) := \limsup_{\delta \to 0} \left( \frac1{\log \delta} \log \left( \int \mu(B(x,\delta)) \der \mu(x) \right) \right) \eqsp.
    \end{align*}
    Similarly, the lower Rényi dimension of $\mu$ is defined by
    \begin{align*}
        \underline{\gamma} (\mu) := \liminf_{\delta \to 0} \left( \frac1{\log \delta} \log \left( \int \mu(B(x,\delta)) \der \mu(x) \right) \right) \eqsp.
    \end{align*}
\end{definition}

\begin{remark}
    The definition above is actually a particular case of the larger family of Rényi dimension, where it actually corresponds to the Rényi dimension of order $2$.
    Note that \Cref{def:correlation-dimension} is also sometimes called the \emph{correlation dimension} in the literature \cite{pesin_rigorous_1993}.
\end{remark}

\begin{remark}
    In particular, if a measure is supported on a smooth manifold and is Ahlfors regular (see \cite{mattila_dimension_2000,falconer2014fractal}) on this manifold, then the Rényi dimension of \Cref{def:correlation-dimension} is equal to the manifold dimension.
\end{remark}

\subsection{Random covariance matrices}
\label{sec:random_matrices_background}

We recall here some basic random matrix theory results related to random covariance matrices and the Marchenko-Pastur theorem. We invite the reader to consult \cite{couillet_random_2022} for additional details.

Given a symmetric matrix $\Sigmabf \in \R^{d \times d}$, we denote its spectrum by $\mathrm{Spec}(\Sigmabf)$ (where the eigenvalues are counted with their multiplicity). The spectral measure of $\Sigma$ is defined by 
\begin{align}
    \label{eq:spectral-measure}
    \widehat{\mu}_{\Sigmabf} := \frac1{d} \sum_{\lambda \in \mathrm{Spec}(\Sigmabf)} \updelta_\lambda \eqsp.
\end{align}
We recall below the celebrated Marchenko-Pastur theorem \cite{marchenko_pastur_1967}. To this end, we first define the Marchenko-Pastur distribution below.

\begin{definition}[Marchenko-Pastur distribution]
    \label{def:mp_dist}
    Let $\psi > 0$. The Marchenko-Pastur distribution with shape parameter $\psi$, denoted here $\mu_{\mathrm{MP}}^{(\psi)}$, is the probability distribution on $\R_+$ defined by
    \begin{align*}
        \mu_{\mathrm{MP}}^{(\psi)}(\der x) := \left( 1 - \frac1{\psi} \right)_+ \updelta_0 + \frac{\sqrt{(\psi^+ - x)(x - \psi^-)}}{2 \pi \psi x} \mathds{1}_{[\psi^-, \psi^+]}(x) \der x \eqsp,
    \end{align*}
    where $y_+ := \max(0,y)$, $\psi^- = |1 - \sqrt{\psi}|^2$, and $\psi^+ := (1 + \sqrt{\psi})^2$.
\end{definition}

The Marchenko-Pastur's theorem is given below.

\begin{theorem}[Marchenko-Pastur's theorem]
    \label{thm:mp_thm}
    Let $X$ be random vector $\Rd$ with i.i.d. components with zero mean, unit variance, and uniformly bounded moments\footnote{We refer to Theorem 2.4 of \cite{couillet_random_2022} for a discussion on the different conditions required by the Marchenko-Pastur theorem.} of order $4 + \epsilon$ for some $\epsilon > 0$. Define the empirical covariance matrix,
    \begin{align*}
        \Sigmabfhat := \frac1{n} \sum_{i=1}^n X_i X_i^T\eqsp,
    \end{align*}
    with $X_1, \dots, X_n$ i.i.d. copies of $X$. Consider that $d = d_n$ with $\lim_{n \to \infty} (d_n / n) = \psi > 0$. Then, with probability one, the empirical measure $\widehat{\mu}_{\Sigmabfhat}$ converges weakly to the Marchenko-Pastur distribution with shape parameter $\psi$. 
\end{theorem}

\section{Omitted Proofs}
\label{sec:omitted-proofs}

In this section, we present the omitted proofs of the results from \Cref{sec:fundamental-limitations,sec:fine-grainded-linear-analysis,sec:preventing-overfitting}.

\subsection{Omitted proofs of \Cref{sec:fundamental-limitations}}
\label{sec:proof-fundamental-limitations}

\subsubsection{Proof of \Cref{lemma:fisher-lemma}.}\label{proof:fisher-lemma}

We present below the proof of \Cref{lemma:fisher-lemma}.

\begin{proof}
    Let $S := (Z_1,\dots,Z_n) \sim \nu^{\otimes n}$ be a dataset sampled from the data distribution.
    Let $(\xtn)_{t \geq 0}$ the Ornstein-Uhlenbeck process $\der \xtn = -\kappa \xtn \der t + \sqrt{2} \der \Bm_t$ initialized at the empirical data distribution $\xfs_0 \sim \frac1{n} \sum_{i=1}^n \updelta_{Z_i}$.
    We denote by $\ptn$ the probability density of $\xtn$, conditioned on $S := (Z_1,\dots,Z_n)$.
    Let $s : [0,T] \times \Rd \to \Rd$ be an arbitrary measurable score function.
    By definition of the Fisher information and the triangle inequality (note that the expectation is over both the noise and the dataset), we have, for any $t > 0$,
    \begin{align*}
        \Eof{\fisher{\ptn}{\pt}} &= \Eof{\normof{\nabla \log \ptn(\xtn) - \nabla \log \pt(\xtn)}^2} \\
         &\leq 2\Eof{\normof{\nabla \log \ptn(\xtn) - s(t, \xtn)}^2} + 2 \Eof{\normof{s(t, \xtn) - \nabla \log \pt(\xtn)}^2}\\
         &= 2\Eof{\normof{\nabla \log \ptn(\xtn) - s(t, \xtn)}^2} + 2 \Eof{\normof{s(t, \xt) - \nabla \log \pt(\xt)}^2}\eqsp,
    \end{align*}
    where the last equality follows as $\law(\xt)$ is the marginal distribution of $\xtn$ under $S \sim \nu^{\otimes n}$.
    We immediately deduce that for any $t > 0$, we have
    \begin{align*}
        \Eof{\fisher{\ptn}{\pt}} \leq 2 \Eof{\LESM (s, t) + \LESMn (s, t)}\eqsp.
    \end{align*}
    For the second part of the statement, by the triangle inequality, we have that 
    \begin{align*}
        \Eof{\fisher{\ptn}{\pt}} = \intrd \Eof{\normof{\nabla \log \frac{\ptn}{\pt}(x)}^2 \ptn(x)} \der x  
        \geq \frac1{2} \Eof{\mathscr{J}(\ptn)} - \mathscr{J}(\pt)  \eqsp.
    \end{align*}
    Note that $\mathscr{J}(p_t) \le \alpha_t^{-2} \mathscr{J}(\nu) <+\infty$ under our assumptions.
    Then, the fact that $\Eof{\mathscr{J}(\ptn)} \to +\infty$ as $t\to 0^+$ follows from standard arguments.    
    For the sake of completeness, let us provide a short proof in the case $\kappa=0$ (without loss of generality, the other cases follow by the same computation by using the invariant distribution of the forward process instead of the Lebesgue measure as a reference measure). 
    Let $h(f) := -\int f(x) \log f(x) \der x$ denote the entropy functional, as soon as it is well-defined.
    As a consequence of the De-Bruijn's identity and the fact that the Fisher information is non-increasing along the semigroup \cite{bakry_analysis_2014}, we have, for $t > s > 0$, that
    \begin{align*}
        h(\hat{p}_t) - h(\hat{p}_s) = \int_s^t \mathscr{J}(\hat{p}_u) \der u \leq (t - s) \mathscr{J}(\hat{p}_s) \eqsp.
    \end{align*}
    Given $(Z_1,\dots,Z_n) \sim \nu^{\otimes n}$, we easily see that
    \begin{align*}
        - h(\hat{p}_s) \geq -\frac1{n} \sum_{i=1}^n h(f_{Z_i, \sigma_s}) - \log(n) = -\frac{d}{2} \log(2 \pi e) + \frac{d}{2} \log(\sigma_s^{-2}) - \log(n) \eqsp,
    \end{align*}
    where $f_{\mu,\sigma}$ is the density of $\Nrm(\mu, \sigma \Irm_d)$. Thus, if we fix $t$ and let $s \to 0^+$, $\Eof{\mathscr{J}(\hat{p}_s)} \to +\infty$.
    
    This concludes the proof.
\end{proof}

\begin{remark}
    As an immediate consequence, for any positive Borel measure $\varpi$ on $[0,T]$, we have
    \begin{align*}
        \Eof{\LESM (s, \varpi) + \LESMn (s, \varpi)} \geq \frac1{2} \intOT \Eof{\fisher{\ptn}{\pt}} \der \varpi(t) \eqsp.
    \end{align*}
\end{remark}

\subsubsection{Lower bound under log-Sobolev inequalities}
\label{sec:fisher-lsi-lower-bound}

In this subsection, we present some additional results to complement \Cref{lemma:fisher-lemma}. Our goal is to show that, under the assumption that the data distribution satisfies a logarithmic Sobolev inequality, we can obtain a similar lower bound to \Cref{lemma:fisher-lemma}, where the relative Fisher information term can be replaced by other divergences, such as the total variation distance.

\paragraph{Conventions for probabilistic inequalities.} We define the total variation distance as
\begin{align*}
    \tv (\mu,\nu) := \sup_A |\mu(A) - \nu(A)|\eqsp.
\end{align*}
We say that a Borel probability distribution $\nu$ satisfies the log-Sobolev inequality with constant $\rho$ if for all differentiable $f \in \Lrm^1(\nu)$, we have
\begin{align*}
    \ent_\nu(f) \leq \frac{\rho}{2} \int \frac{\normof{\nabla f}^2}{f} \der \nu\eqsp.
\end{align*}
Equipped with these definitions, we have the following proposition.
\begin{proposition}
    \label{eq:integrated-tv-lower-bound-lsi}
    Assume that the data distribution satisfies the log-Sobolev inequality with constant $\rho_0$. Let $\varpi$ be a positive Borel measure on $[0,T]$, we have 
    \begin{align*}
        \Eof{\LESM (s, \varpi) + \LESMn (s, \varpi)} \geq \int_{[0,T]} \frac{2}{\alpha_t^2 \rho_0 + \sigma_t^2} \Eof{\mathrm{TV} (\pt, \ptn)^2} \der \varpi(t) \eqsp.
    \end{align*}
\end{proposition}
\begin{proof}
    By stability of the log-Sobolev inequality under Lipschitz mappings and convolutions \cite{chafai_entropies_2004}, we have that $\pt$ satisfies the log-Sobolev inequality with constant $ \rho_t := \alpha_t^2 \rho_0 + \sigma_t^2$.
    Therefore, for any $t>0$, we have
    \begin{align*}
        \frac1{2}\fisher{\ptn}{\pt} &\geq \frac1{\alpha_t^2 \rho_0 + \sigma_t^2} \klb{\ptn}{\pt} \by{log-Sobolev inequality} \\
        &\geq \frac{2}{\alpha_t^2 \rho_0 + \sigma_t^2} \tv(\ptn, \pt)^2 \by{Pinsker's inequality}\eqsp.
    \end{align*}
    The result immediately follows by integrating over $\varpi$ and applying \Cref{lemma:fisher-lemma}.
\end{proof}

\begin{remark}
We immediately see that, if $\kappa > 0$, we have,
\begin{align*}
    \Eof{\LESM (s, \varpi) + \LESMn (s, \varpi)} \geq \frac{2}{\max(\kappa^{-1}, \rho_0)} \int_{[0,T]}  \Eof{\mathrm{TV} (\pt, \ptn)^2} \der \varpi(t) \eqsp.
\end{align*}
\end{remark}

\subsubsection{Proof of \Cref{prop:negative_sm_loss}}
\label{sec:proof-negative-sm-result}

Before giving the proof of \Cref{prop:negative_sm_loss}, we provide some technical lemmas below.
Recall that we denote by $p_{t|0}(\cdot|x)$ the conditional density of $X_t$ given $X_0 = x$, defined by \Cref{eq:cond_score_form}.

\begin{lemma}
\label{lem:gauss_concentration}
Given any $x \in \Rd$ and $Y \sim p_{\epsilon|0}(\cdot|x)$, we have that for any $\Delta \geq 0$,
\begin{gather*}
    \prob(\|Y - \alpha_\epsilon x\| \geq \sigma_\epsilon m + \Delta) \leq \exp \bigg ( - \frac{\Delta^2}{2\sigma_\epsilon^2} \bigg )\eqsp,\\
    \prob(\|Y - \alpha_\epsilon x\| \leq \sigma_\epsilon m - \Delta) \leq \exp \bigg ( - \frac{\Delta^2}{2\sigma_\epsilon^2} \bigg )\eqsp,
\end{gather*}
where $m = \E_{Z \sim \Nrm(0, \Irm_d)}[\|Z\|]$.
\end{lemma}
\begin{proof}
By definition, we have that the random variable $Z = \sigma_\epsilon^{-1} (Y - \alpha_\epsilon x)$ is a standard multivariate Gaussian. Thus, by Gaussian concentration of measure (see \eg, Section 5.4.2 of \cite{bakry_analysis_2014}), we have that for any 1-Lipschitz function $f$,
\begin{equation*}
    \prob(f(Z) - \E[f(Z)] \geq t) \leq \exp(-t^2/2)\eqsp.
\end{equation*}
Choosing $f(z) = \|z\|$ and $f(z) = -\|z\|$ along with $t = \Delta/\sigma_\epsilon$.
\end{proof}

The proof of \Cref{prop:negative_sm_loss} is based on the following proposition, providing a lower bound on the sample size, provided that two KL divergences associated to the population and empirical forward processes are both smaller than an absolute constant. Below, we use the following convention and notation for the total variation distance:
\begin{align}
    \label{eq:tv_distance}
    \normof{\mu - \nu}_{\mathrm{TV}} := \sup_A |\mu(A) - \nu(A)| \eqsp,
\end{align}
for probability measure $\mu$ and $\nu$. We also recall that we denote by $p_t$ the probability density of $X_t$ following \Cref{eq:forward-process-sde} and by $\hat{p}_t$ the probability density of the process $\der \xtn = -\kappa \xtn \der t + \sqrt{2} \der \Bm_t$ initialized at $\xfs_0 \sim \frac1{n} \sum_{i=1}^n \updelta_{Z_i}$, for $t > 0$, where $S := (Z_1, \dots,Z_n) \sim \nu^{\otimes n}$ is the dataset.

\begin{proposition}
\label{prop:negative_kl}
Let $\epsilon>0$. 
For any distribution $\hat{\nu}$ on $\R^d$, which might depend on the dataset $S$, if $\Eof{\klb{\hat{p}_\epsilon}{\hat{\nu}}} \leq 1 / 16$ and $\Eof{\klb{{p}_\epsilon}{\hat{\nu}}} \leq 1 / 16$, then we have $n \geq n_{\min}(\epsilon)$, with, for $\epsilon$ small enough
\begin{align*}
     \log n_{\min} (\epsilon) \geq -\frac1{2} \log \left( 8 \sqrt{d \vee \log(64)} \sigma_\epsilon \right) \underline{\gamma}(\nu) \eqsp,
\end{align*}
where $\underline{\gamma}(\nu)$ is the lower Rényi dimension defined in \Cref{def:correlation-dimension}.
\end{proposition}

\begin{proof}[Proof of Proposition \ref{prop:negative_kl}]
We begin with Pinsker's inequality and the reverse triangle inequality for the total variation distance, which are used to obtain
\begin{align*}
    \Eof{\klb{{p}_\epsilon}{\hat{\nu}}} &\geq 2 \E [ \|p_\epsilon - \hat{\nu}\|_{\operatorname{TV}}^2 ]\\
    &\geq 2 \E \Big [ (\|p_\epsilon - \hat{p}_\epsilon\|_{\operatorname{TV}} - \|\hat{p}_\epsilon - \hat{\nu}\|_{\operatorname{TV}})^2 \Big ]\\
    &\geq 2 \bigg ( \sqrt{\Eof{\|p_\epsilon - \hat{p}_\epsilon\|^2_{\operatorname{TV}}}} - \sqrt{ \Eof {\|\hat{p}_\epsilon - \hat{\nu}\|_{\operatorname{TV}}^2}} \bigg )^2 \eqsp,
\end{align*}
where the final inequality follows from the triangle inequality in weighted $\Lrm^2$-norm. Since, we have $2\E [\|\hat{p}_\epsilon - \hat{\nu}\|_{\operatorname{TV}}^2] \leq \Eof{\klb{\hat{p}_\epsilon}{\hat{\nu}}} \leq 1/16$, we conclude from this that if $\Eof{\klb{{p}_\epsilon}{\hat{\nu}}}\leq 1/16$ then $\E [\|p_\epsilon - \hat{p}_\epsilon\|_{\operatorname{TV}}^2] \leq \frac{1}{8}$.

We now identify a test set to lower bound the total variation using that for any Borel set $A$,
\begin{equation*}
    \|p_\epsilon - \hat{p}_\epsilon\|_{\operatorname{TV}} \geq |p_\epsilon(A) - \hat{p}_\epsilon(A)| \eqsp.
\end{equation*}
As in \Cref{lem:gauss_concentration}, we define $m = \E_{Z \sim \Nrm(0, \Irm_d)}[\|Z\|]$.
We define the set $A_R := \{y \in \R^d: y \notin B_{\sigma_\epsilon m + R}(\alpha_\epsilon Z_i), \forall i =1,\dots,n\}$ so that the empirical density is upper bounded by
\begin{equation*}
    \hat{p}_\epsilon(A_R) \leq \frac{1}{N} \sum_{i=1}^N p_{\epsilon|0}(B_{m\sigma_\epsilon + R}(\alpha_\epsilon Z_i)^\complement|Z_i) = 1- p_{\epsilon, R} \eqsp,
\end{equation*}
where $p_{\epsilon, R} := p_{\epsilon|0}(B_{m\sigma_\epsilon + R}(\mathbf{0})|\mathbf{0})$. Define the set $$\tilde{A}_R := \{y \in \R^d: \|y - Z_i\| \geq 2 (R + m \sigma_\epsilon)/\alpha_\epsilon, \forall i=1,\dots,n\} \eqsp.$$ The population density is lower bounded using that for any $x \in \tilde{A}_R$, we have $B_{m\sigma_\epsilon + R}(\alpha_\epsilon x) \subset A_R$ and hence,
\begin{align*}
    p_\epsilon(A_R) &\geq \int_{\tilde{A}_R} p_{\epsilon|0}(A_R|x) \nu(\der x)\\
    &\geq \int_{\tilde{A}_R} p_{\epsilon|0}(B_{m\sigma_\epsilon + R}(\alpha_\epsilon x)|x) \nu(\der x)\\
    &= \nu(\tilde{A}_R) p_{\epsilon, R} \eqsp.
\end{align*}
We further bound this using the union bound,
\begin{align*}
    \E_S[\nu(\tilde{A}_R)] &=1 - \Eof[S]{\Pof[X \sim \nu]{\bigcup_{i=1}^n \{\|X - x_i\| < {2(R + m \sigma_\epsilon)}/{\alpha_\epsilon} \}}}\\
    &\geq \left(1 - \sum_{i=1}^n \E[\prob_{X \sim \nu}(\|X - x_i\| < 2 (R + m \sigma_\epsilon)/\alpha_\epsilon)]\right)_+\\
    &= (1 - n \delta(2 (R + m \sigma_\epsilon)/\alpha_\epsilon))_+ \eqsp,
\end{align*}
where we define $\delta(r) = \E_{X \sim \nu}[\nu(B_{r}(X))]$. The total variation is then lower bounded using $A_R$ as a test set:
\begin{equation*}
    \E[\|p_\epsilon - \hat{p}_\epsilon\|_{\operatorname{TV}}] \geq \E[p_\epsilon(A_R) - \hat{p}_\epsilon(A_R)] \geq \left( \left(1 - n \delta \left(\frac{2 (R + m \sigma_\epsilon)}{\alpha_\epsilon} \right)\right)_+ p_{\epsilon, R} - (1 - p_{\epsilon, R}) \right)_+ \eqsp.
\end{equation*}
To guarantee that this bound is non-trivial, we now choose $R$ using Lemma \ref{lem:gauss_concentration}. Setting $R^2 = 2 \log(8) \sigma_\epsilon^2$, leads to the bound,
\begin{equation*}
    p_{\epsilon, R} \geq 1 - \exp \bigg ( - \frac{R^2}{2 \sigma_\epsilon^2} \bigg ) \geq \frac{3}{4} \eqsp.
\end{equation*}
We also choose $\epsilon$ such that $\alpha_\epsilon \geq 1/2$ so that $2(R + m\sigma_\epsilon)/\alpha_\epsilon \leq 8 \sqrt{d \vee \log(64)} \sigma_\epsilon$ and $8 \sqrt{d \vee \log(64)} \sigma_\epsilon < 1$.

Thus, by Jensen's inequality, we have the lower bound,
\begin{equation*}
    \E[\|p_\epsilon - \hat{p}_\epsilon\|_{\operatorname{TV}}^2]^{1/2} \geq \frac{1}{4} \Big ( 3 (1 - n \delta(8 \sqrt{d \vee \log(64)} \sigma_\epsilon))_+ - 1 \Big )_+ \eqsp.
\end{equation*}
From this bound we deduce that,
\begin{align*}
    \Eof{\|p_\epsilon - \hat{p}_\epsilon\|_{\text{TV}}^2} \leq \frac{1}{8} &\implies ( 3 (1 - n \delta(8 \sqrt{d \vee \log(64)} \sigma_\epsilon))_+ \leq 1 + \sqrt{2}\\
    & \implies n \geq \frac{2 - \sqrt{2}}{3 \, \delta(8 \sqrt{d \vee \log(64)} \sigma_\epsilon)} =: n_{\min}(\epsilon) \eqsp.
\end{align*}
By the definition of the lower Rényi dimension in \Cref{def:correlation-dimension}, we have that
\begin{align*}
    \underline{\gamma}(\nu) = \liminf_{\epsilon \to 0^+} \frac{\log \delta (8 \sqrt{d \vee \log(64)} \sigma_\epsilon)}{\log (8 \sqrt{d \vee \log(64)} \sigma_\epsilon)} \eqsp.
\end{align*}
Let $\eta > 0$ be fixed, by definition, we have that there exists $\epsilon_0$ such that, for all $\epsilon < \epsilon_0$, we have
\begin{align*}
    \log n_{\min} (\epsilon) \geq  \log \left( \frac{2 - \sqrt{2}}{3} \right) - \log \left( 8 \sqrt{d \vee \log(64)} \sigma_\epsilon \right) (\underline{\gamma}(\nu) - \eta) \eqsp.
\end{align*}
In particular, if $\underline{\gamma}(\nu) > 0$, for $\epsilon$ small enough, we have that
\begin{align*}
     \log n_{\min} (\epsilon) \geq -\frac1{2} \log \left( 8 \sqrt{d \vee \log(64)} \sigma_\epsilon \right) \underline{\gamma}(\nu) \eqsp.
\end{align*}
Moreover, if $\underline{\gamma}(\nu) = 0$, then the above inequality is always verified.
This concludes the proof.
\end{proof}

We can now present the proof of \Cref{prop:negative_sm_loss}.

\begin{proof}[Proof of Theorem \ref{prop:negative_sm_loss}]
Suppose that $\E_{X \sim \nu}[\|X\|^2] = \sigma^2 < \infty$, $T \geq 1 + \frac{1}{2\kappa} \log (32 (\sigma^2 + \kappa^{-1} d))$, and $d \geq 5$. We begin by expressing the weighted ESM loss in terms of the standard ESM loss by using the fact that the density $w$ of $\varpi$ is non-decreasing on its support. We have, for $\delta \geq \epsilon$,
\begin{align*}
    \LESMn(\hat{s}, \varpi) &= \int_\epsilon^T \LESMn(\hat{s}, t) w(t) \der t \geq  w(\delta) \LESMn(\hat{s}, \varpi_{\mathrm{ELBO}}^\delta) \eqsp,
\end{align*}
where we recall $\delta = w^{-1}(32 \varepsilon)$. By the assumption that $\varepsilon \leq \sup(w) / 32$, we have,
\begin{equation*}
     \E[\LESMn(\hat{s}, \varpi_{\mathrm{ELBO}}^\delta)] \leq \frac{\varepsilon}{w(\delta)} \leq \frac{1}{32} \eqsp.
\end{equation*}
Via the same argument, we also obtain that $\E[\LESM(\hat{s}, \varpi)] \leq \varepsilon$ implies that $\E[\LESM(\hat{s}, \varpi_{\mathrm{ELBO}}^\delta)] \leq 1/32$. Therefore it is sufficient to show that $\E[\LESMn(\hat{s}, \varpi_{\mathrm{ELBO}}^\delta)] \leq 1/32,~ \E[\LESM(\hat{s}, \varpi_{\mathrm{ELBO}}^\delta)] \leq 1/32$ leads to the lower bound on $n$.

It follows from standard bounds based on Girsanov's theorem that
\begin{align*}
    \klb{p_\delta}{q_{T - \delta}} \leq \LESM(\hat{s}, \varpi_{\mathrm{ELBO}}^\delta) + \klb{p_T}{q_0} \eqsp,\\
    \klb{\hat{p}_\delta}{q_{T - \delta}} \leq \LESMn(\hat{s}, \varpi_{\mathrm{ELBO}}^\delta) + \klb{\hat{p}_T}{q_0} \eqsp.
\end{align*}
Using that $q_0$ satisfies a logarithmic Sobolev inequality with constant $\kappa^{-1}$, in combination with Corollary 2 of \cite{otto_comment_2001}, we obtain the upper bound,
\begin{align*}
    \klb{p_T}{q_0} &\leq \exp(-2\kappa(T-s)) \klb{p_s}{q_0}\\
    &\leq \frac{\exp(-2\kappa(T-s))}{4s} \wass_2(p_0, q_0)^2 \eqsp,
\end{align*}
any $s \in (0, T]$. Setting $s = 1$, we obtain,
\begin{equation*}
    \klb{p_T}{q_0} \leq \frac{\exp(-2 \kappa (T-1)) }{2} (\sigma^2 + \kappa^{-1} d) \eqsp.
\end{equation*}
By the lower bound assumption on $T$, we have that $\klb{p_T}{q_0} \leq 1/32$ and, via a similar argument, $\E[ \klb{\hat{p}_T}{q_0} ] \leq 1/32$ also. Thus, it follows that
\begin{equation*}
    \Eof{\klb{p_\delta}{q_{T - \delta}}},~\Eof{\klb{\hat{p}_\delta}{q_{T - \delta}}}  \leq 1/16 \eqsp.
\end{equation*}
The result then follows immediately from Proposition \ref{prop:negative_kl}.
\end{proof}

\subsection{Omitted proofs of \Cref{sec:fine-grainded-linear-analysis}}
\label{sec:proofs-fine-grainded-linear-analysis}

In this section, we present the proofs of \Cref{lemma:fixed-time-non-asymptotic-lemma} and \Cref{prop:fixed-time-asymptotic-linear-results}.

\subsubsection{Proof of \Cref{lemma:fixed-time-non-asymptotic-lemma}}

We present below the proof of \Cref{lemma:fixed-time-non-asymptotic-lemma}.

\begin{proof}
    Let $(Z_1,\dots,Z_n) \in (\Rd)^N$ be a dataset and $W \in \R^{p \times d}$ be fixed. By assumption, $W$ has full rank.
    Let $\Xi \sim \Nrm(0, \Irm_d)$. We can write the empirical risk as
    \begin{align*}
        A \mapsto \frac1{n} \sum_{i=1}^n \Eof{\normof{\sigma_t AW (\alpha_t Z_i + \sigma_t \Xi) + \Xi}^2} \eqsp.
    \end{align*}
    As $\Eof{\Xi} = 0$, we have
    \begin{align*}
       \sigma_t^2 \LDSMn(s_{W,A}, t) &= d + \sigma_t^2 \frac1{n} \sum_{i=1}^n \Eof{\trace \left( W^T A^T A W (\alpha_t Z_i + \sigma_t \Xi) (\alpha_t Z_i + \sigma_t \Xi)^T \right)} \\&\qquad + 2 \sigma_t^2 \Eof{\trace \left( AW \Xi \Xi^T \right) } \\
        &= d + \sigma_t^2 \trace \left( A^T A W (\alpha_t^2 \widehat{\Sigmabf} + \sigma_t^2 \Irm_d) W^T \right) + 2 \sigma_t^2 \trace (AW) \\
        &= d + \sigma_t^2 \trace \left( A^T A W \widehat{\Sigmabf}_t W^T  \right) + 2 \sigma_t^2 \trace (AW) \eqsp.
    \end{align*}
    Let $A \in \argmin   \widehat{\Rcal}_{W,t} (\widehat{A}) $, by taking the gradient above, we have
    \begin{align}
        \label{eq:gradient_zero_linear_net}
        A W \widehat{\Sigmabf}_t W^T + W^T = 0 \eqsp.
    \end{align}

    \textbf{Case 1} ($p < d$). 
    By our assumptions, the matrix $W$ has full rank, so its rank is $p$ in this case.
    The matrix $\widehat{\Sigmabf}_t$ is always invertible, as soon as $t > 0$; we also know that $\rank (W) = \rank (W \widehat{\Sigmabf}_t^{1 / 2})$, therefore, we conclude that $\rank (W \widehat{\Sigmabf}_t W^T) = p$, so the matrix $W \widehat{\Sigmabf}_t W^T \in \R^{p\times p}$ is invertible. 

    Therefore, the empirical risk minimizer is given by
    \begin{align*}
        \widehat{A} := - W^T (W \widehat{\Sigmabf}_t W^T)^{-1} \eqsp.
    \end{align*}
    Using \Cref{eq:gradient_zero_linear_net}, we have that
    \begin{align*}
       \sigma_t^2 \LDSMn(s_{W,\Ahat}, t)  = d - \sigma_t^2 \trace \left( W^T (W \widehat{\Sigmabf}_t W^T)^{-1} W \right) \eqsp.
    \end{align*}
    To compute the score matching loss, we note that when $\ora{X}_0 \sim \nu = \Nrm(0,\Sigmabf)$, then $\xt \sim \Nrm(0, \alpha_t^2 \Sigma + \sigma_t^2 \Irm_d) = \Nrm(0, \Sigmabf_t)$. Therefore, the score can be expressed as
    \begin{align*}
        \nabla \log \pt(x) = -\Sigmabf_t^{-1} x \eqsp.
    \end{align*}
    The score matching loss is then equal to
    \begin{align*}
        \LESM (s_{W,\widehat{A}}, t) = \Eof{\normof{\widehat{A} W (\alpha_t Z + \sigma_t \Xi) + \Sigmabf_t^{-1}  (\alpha_t Z + \sigma_t \Xi)}^2} \eqsp,
    \end{align*}
    with $(Z,\Xi) \sim \nu \otimes \Nrm (0, \Irm_d)$. By symmetry of $\widehat{A}W$, we obtain that
    \begin{align*}
        \LESM (s_{W,\widehat{A}}, t) &= \Eof{\normof{\widehat{A} W (\alpha_t Z + \sigma_t \Xi) + \Sigmabf_t^{-1}  (\alpha_t Z + \sigma_t \Xi)}^2} \\
         &= \trace \left( W^T \Ahat^T \Ahat W \Sigmabf_t + 2 \Ahat W + \Sigmabf^{-1}_t \right) \\
        &= \frobnorm{\left(W^T (W \widehat{\Sigmabf}_t W^T)^{-1} W - \Sigmabf_t^{-1} \right) \Sigmabf_t^{1/2}}^2 \eqsp.
    \end{align*}

    \textbf{Case 2.} ($p\geq d$) In this case, by similar argument as in the first case, we note that the matrix $W \in \R^{p \times d}$ has full rank by assumption. Consider the linear map $\Phi_W : \R^{d \times p} \to \R^{d \times d}$ defined by $\Phi_W(A) = AW$. Then, we easily show that this map is almost-surely surjective, because the matrix $W$ admits a left inverse $W^+$ such that $W^+ W = \Irm_d$ (which can be obtained from the singular value decomposition of $W$).

    As both the empirical risk and the score matching loss depend only on $A$ through the product $AW$, we have that all the empirical risk minimizers $\widehat{A}$ satisfy $\widehat{A} W = \widehat{H} \in \R^{d \times d}$, with $\widehat{H}$ the empirical minimizer of the following linear model:
    \begin{align*}
        \widehat{H} \in \argmin_{H} \setof{\frac1{n} \sum_{i=1}^n \Eof{\normof{\sigma_t H (\alpha_t Z_i + \sigma_t \Xi) + \Xi} ^2}} \eqsp, \quad (Z,\Xi) \sim \nu \otimes \Nrm (0, \Irm_d) \eqsp.
    \end{align*}
    By similar computations as before, we easily show that $\widehat{H} = -\widehat{\Sigmabf}_t^{-1}$. 
    Therefore, the minimum empirical risk is
    \begin{align*}
         \sigma_t^2 \LDSMn(s_{W,\Ahat}, t)  = d - \sigma_t^2 \trace \left( \widehat{\Sigmabf}_t^{-1} \right) \eqsp.
    \end{align*}
    Finally, the score matching loss can be expressed as
    \begin{align*}
        \LESM (s_{W,\widehat{A}}, t) &= \trace \left( (\widehat{A} W + \Sigmabf_t^{-1})^2 \Sigmabf_t \right) = \frobnorm{\left(\widehat{\Sigmabf}_t^{-1} - \Sigmabf_t^{-1} \right) \Sigmabf_t^{1/2}}^2 \eqsp.
    \end{align*}
    This concludes the proof.
\end{proof}

\subsubsection{Proof of \Cref{prop:fixed-time-asymptotic-linear-results}}

We prove \Cref{prop:fixed-time-asymptotic-linear-results} as a corollary of the following more general result, which gives the exact asymptotic behavior of the score matching loss and the empirical risk.

\begin{proposition}
\label{prop:linear_net_exact_learning_curves}
    Consider the same setup as in \Cref{prop:fixed-time-asymptotic-linear-results}. Let $\Ahat \in \argmin_A \widehat{\Rcal}_{W,t} (A) $. If $\psi_p > \psi_d$, we have
    \begin{align*}
        \lim_{n \to \infty} \frac1{d} \inf_A  \left( \sigma_t^2 \LDSMn(s_{W,A}, t) \right) = 1 - \sigma_t^2 \int \frac{\der \mu_{\mathrm{MP}}^{(\psi_d)} (\lambda)}{\alpha_t^2 \beta \lambda + \sigma_t^2} \eqsp,
    \end{align*}
    and
    \begin{align*}
         \lim_{n \to\infty} \LESM (s_{W,\widehat{A}}, t) = (\alpha_t^2 \beta + \sigma_t^2) \int \left( \frac1{\alpha_t^2 \beta \lambda + \sigma_t^2} - \frac1{\alpha_t^2 \beta + \sigma_t^2}  \right)^2 \der \mu_{\mathrm{MP}}^{(\psi_d)} (\lambda) \eqsp.
    \end{align*}
    If $\psi_p < \psi_d$, we have
    \begin{align*}
        \frac1{d}  \sigma_t^2 \LDSMn(s_{W,\Ahat}, t) \underset{n\to \infty}{\longrightarrow} 1 - \sigma_t^2 \frac{\psi_p}{\psi_d}  \int \frac1{\alpha_t^2 \beta \lambda + \sigma_t^2} \der \mu_{\mathrm{MP}}^{(\psi_p)} (\lambda) \eqsp,
    \end{align*}
    and
    \begin{align*}
         \frac1{d} \LESM (s_{W,\widehat{A}}, t) \underset{n\to\infty}{\longrightarrow} (\alpha_t^2 \beta + \sigma_t^2) \frac{\psi_p}{\psi_d} \int \left( \frac1{\alpha_t^2 \beta \lambda + \sigma_t^2} - \frac1{\alpha_t^2 \beta + \sigma_t^2}  \right)^2 \der \mu_{\mathrm{MP}}^{(\psi_p)} (\lambda) + \frac{1 - \frac{\psi_p}{\psi_d}}{\alpha_t^2 \beta + \sigma_t^2} \eqsp.
    \end{align*}
\end{proposition}

\begin{proof}
    First, we note that in that case we have $\Sigmabf_t = (\alpha_t^2 \beta + \sigma_t^2) \Irm_d$. 
    Moreover, we observe that, as the matrix $W \in \R^{p \times d}$ has i.i.d. standard Gaussian entries, it has almost-surely full rank.
    
    We can therefore apply \Cref{lemma:fixed-time-non-asymptotic-lemma}, we distinguish two cases.

    \textbf{Case 1.} ($\psi_p > \psi_d$) For $n$ large enough, we have $p > d$. In this case, by \Cref{lemma:fixed-time-non-asymptotic-lemma}, the normalized empirical risk can be written as
    \begin{align*}
        \frac1{d} \inf_A \left( \sigma_t^2 \LDSMn(s_{W,A}, t) \right)  &= 1 - \sigma_t^2 \int \frac{\der \mu_{\beta^{-1} \widehat{\Sigmabf}} (\lambda)}{\alpha_t^2 \beta \lambda + \sigma_t^2} \eqsp,
    \end{align*}
    where, for a positive definite matrix $M$, $\muhat_{M}$ denotes the empirical spectral measure of $M$, see \Cref{sec:random_matrices_background}.

    By the Marchenko-Pastur theorem (e.g. Theorem 2.4 of \cite{couillet_random_2022}), we know that, almost-surely, the measure $\mu_{\beta^{-1} \widehat{\Sigmabf}}$ converges weakly to the Marchenko-Pastur distribution with shape parameter $\psi_d$, denoted $\mu_{\mathrm{MP}}^{(\psi_d)}$, see \Cref{sec:random_matrices_background}.
    The map $\lambda \mapsto (\alpha_t^2 \beta \lambda + \sigma_t^2)$ is bounded and continuous, therefore, we have
    \begin{align*}
        \lim_{n \to \infty} \frac1{d} \inf_A  \left( \sigma_t^2 \LDSMn(s_{W,A}, t) \right) = 1 - \sigma_t^2 \int \frac{\der \mu_{\mathrm{MP}}^{(\psi_d)} (\lambda)}{\alpha_t^2 \beta \lambda + \sigma_t^2} \eqsp.
    \end{align*}
    Similarly, we have, with $\Ahat \in \argmin_A \widehat{\Rcal}_{W,t} (A) $, 
    \begin{align*}
         \LESM (s_{W,\widehat{A}}, t) &= \frac1{d}(\alpha_t^2 \beta + \sigma_t^2) \frobnorm{\left(\widehat{\Sigmabf}_t^{-1} - \Sigmabf_t^{-1} \right) \Sigmabf_t^{1/2}}^2 \\
         &\underset{n \to \infty}{\longrightarrow}  (\alpha_t^2 \beta + \sigma_t^2) \int \left( \frac1{\alpha_t^2 \beta \lambda + \sigma_t^2} - \frac1{\alpha_t^2 \beta + \sigma_t^2}  \right)^2 \der \mu_{\mathrm{MP}}^{(\psi_d)} (\lambda) \eqsp.
    \end{align*}

    \textbf{Case 2.} ($\psi_p < \psi_d$) For $n$ large enough, we have $d > p$ In this case, we note (see the proof of \Cref{lemma:fixed-time-non-asymptotic-lemma}) that the matrix $W$ is almost surely of rank $p$. Therefore, we can write the singular value decomposition of $W$ in the form $W = U \Lambda V^T$, where $U \in \R^{p \times p}$ is orthonormal, $\Lambda \in \R^{p \times p}$ is a diagonal matrix with positive diagonal coefficients, and $V \in \R^{d \times p}$ is column orthonormal ($V^T V = \Irm_p$). 

    With these notations and noting that $V^T V = \Irm_p$, we have, almost-surely, that
    \begin{align*}
       \frac1{d}  \sigma_t^2 \LDSMn(s_{W,\Ahat}, t) &= 1 - \frac{\sigma_t^2}{d} \trace \left( V \Lambda U^T(U\Lambda V^T \Sigmabfhat_t V \Lambda U^T)^{-1} U \Lambda V^T \right) \\
       &= 1 - \frac{\sigma_t^2}{d} \trace \left( V ( V^T \Sigmabfhat_t V )^{-1}   V^T \right)\\
       &= 1 - \frac{\sigma_t^2}{d} \trace \left( ( V^T \Sigmabfhat_t V )^{-1}  \right) \eqsp. 
    \end{align*}
    Now we observe that
    \begin{align*}
        V^T\Sigmabfhat_t V = \alpha_t^2 V^T \Sigmabfhat V + \sigma_t^2 \Irm_p = \frac{\alpha_t^2}{n} \sum_{i=1}^n V^T Z_i (V^T Z_i)^T + \sigma_t^2 \Irm_p \eqsp. 
    \end{align*}
    Recall that $(Z_1,\dots, Z_n) \sim \Nrm(0, \beta \Irm_d)^{\otimes n}$ and that $V$ is column orthonormal (actually, it is even uniformly distributed) and independent of $(Z_1,\dots, Z_n)$. By independence and invariance by rotation of $\Nrm(0, \beta \Irm_d)$, we have that $V^T Z_i \sim \Nrm(0, \beta\Irm_p)$. Moreover, we have that $(V^T Z_1,\dots, V^T Z_n) \sim \Nrm(0, \beta \Irm_p)^{\otimes n}$. Thus, we can apply the $p$-dimensional version of the Marchenko-Pastur theorem, which gives that
    \begin{align*}
        \frac1{d}  \sigma_t^2 \LDSMn(s_{W,\Ahat}, t) \underset{n\to \infty}{\longrightarrow} 1 - \sigma_t^2 \frac{\psi_p}{\psi_d}  \int \frac1{\alpha_t^2 \beta \lambda + \sigma_t^2} \der \mu_{\mathrm{MP}}^{(\psi_p)} (\lambda) \eqsp.
    \end{align*}
    By \Cref{lemma:fixed-time-non-asymptotic-lemma}, we have
    \begin{align*}
       \frac1{d} \LESM (s_{W,\widehat{A}}, t) &= \frac1{d} \frobnorm{\left(W^T (W \widehat{\Sigmabf}_t W^T)^{-1} W - \Sigmabf_t^{-1} \right) \Sigmabf_t^{1/2}}^2\\
       &=  \frac{\alpha_t^2 \beta + \sigma_t^2}{d} \frobnorm{\left(V ( V^T \Sigmabfhat_t V )^{-1}   V^T - \Sigmabf_t^{-1} \right) }^2\eqsp.
    \end{align*}
    Recall that $\Sigmabf_t = (\alpha_t^2 \beta + \sigma_t^2) \Irm_d$.
    By the Pythagorean theorem, we have
     \begin{align*}
       \frac1{d} \LESM (s_{W,\widehat{A}}, t) &= \frac{\alpha_t^2 \beta + \sigma_t^2}{d} \left( \frobnorm{V ( V^T \Sigmabfhat_t V )^{-1} V^T - \frac{VV^T}{\alpha_t^2 \beta + \sigma_t^2}  }^2  + \frac{\frobnorm{ \Irm_d - VV^T }^2}{(\alpha_t^2 \beta + \sigma_t^2)^2} \right) \\
       &= \frac{\alpha_t^2 \beta + \sigma_t^2}{d} \left( \frobnorm{ ( V^T \Sigmabfhat_t V )^{-1}  - \frac{\Irm_p}{\alpha_t^2 \beta + \sigma_t^2}  }^2  + \frac{\trace \left( \Irm_d - V V^T \right)}{(\alpha_t^2 \beta + \sigma_t^2)^2} \right) \\
       &= \frac{\alpha_t^2 \beta + \sigma_t^2}{d} \left( \frobnorm{ ( V^T \Sigmabfhat_t V )^{-1}  - \frac{\Irm_p}{\alpha_t^2 \beta + \sigma_t^2}  }^2  + \frac{d - p}{(\alpha_t^2 \beta + \sigma_t^2)^2} \right) \eqsp.
    \end{align*}
    By the arguments above (for the empirical risk), we observe that we can apply the $p$ dimensional version of the Marchenko-Pastur theorem to the first term, this gives, almost-surely,
    \begin{align*}
         \frac1{d} \LESM (s_{W,\widehat{A}}, t) \underset{n\to\infty}{\longrightarrow} (\alpha_t^2 \beta + \sigma_t^2) \frac{\psi_p}{\psi_d} \int \left( \frac1{\alpha_t^2 \beta \lambda + \sigma_t^2} - \frac1{\alpha_t^2 \beta + \sigma_t^2}  \right)^2 \der \mu_{\mathrm{MP}}^{(\psi_p)} (\lambda) + \frac{1 - \frac{\psi_p}{\psi_d}}{\alpha_t^2 \beta + \sigma_t^2} \eqsp.
    \end{align*}
    This concludes the proof.
\end{proof}

We present below the proof of \Cref{prop:fixed-time-asymptotic-linear-results}.

\begin{proof}(of \Cref{prop:fixed-time-asymptotic-linear-results})
We apply \Cref{prop:linear_net_exact_learning_curves} and focus on the case where $\psi_p \wedge\psi_d > 1$. Then, in both cases below, the Marchenko-Pastur distributions appearing in \Cref{prop:linear_net_exact_learning_curves} have a mass at zero, which we exploit to obtain our estimates. Let $f_{\mathrm{MP}}^{(\psi)}$ denote the bulk density of the Marchenko-Pastur distribution with shape parameter $\psi > 0$. Note that this is only a probability density when $\psi \leq 1$, see \Cref{sec:random_matrices_background}.

\textbf{Case 1.} ($\psi_p > \psi_d$). Then, we have, almost-surely,
\begin{align*}
       \lim_{n \to \infty}\frac1{d} \inf_A  \left( \sigma_t^2 \LDSMn(s_{W,A}, t) \right) = 1 - \sigma_t^2 \left(1 - \frac1{\psi_d} \right) \frac1{\sigma_t^2} - \sigma_t^2 \int \frac{f_{\mathrm{MP}}^{(\psi_d)}(\lambda)}{\alpha_t^2 \beta \lambda + \sigma_t^2} \der \lambda \eqsp. 
\end{align*}
The bulk density $f_{\mathrm{MP}}^{(\psi_d)}$ is supported on $[(\sqrt{\psi_d} - 1)^2, (\sqrt{\psi_d} + 1)^2]$ and we have $\int f_{\mathrm{MP}}^{(\psi_d)}(\lambda) \der \lambda = \psi_d^{-1}$. Therefore, we have, almost-surely
\begin{align*}
       \lim_{n \to \infty}\frac1{d} \inf_A  \left( \sigma_t^2 \LDSMn(s_{W,A}, t) \right) = \frac1{\psi_d} + \landau{\frac1{\psi_d^2}} \eqsp,
\end{align*}
as $\psi_d \to \infty$.
For the score matching loss, we have similarly that, almost-surely,
\begin{align*}
    \lim_{n \to \infty} \left(\frac1{d} \LESM (s_{W,\widehat{A}}, t) \right) &= \left( 1 - \frac1{\psi_d} \right) \frac{\alpha_t^4 \beta^2}{\sigma_t^4(\alpha_t^2 \beta +  \sigma_t^2)} \\& \qquad+ 
    (\alpha_t^2 \beta + \sigma_t^2) \int \left( \frac1{\alpha_t^2 \beta \lambda + \sigma_t^2} - \frac1{\alpha_t^2 \beta + \sigma_t^2}  \right)^2 f_{\mathrm{MP}}^{(\psi_d)} (\lambda) \der \lambda \\
    &= \left( 1 - \frac1{\psi_d} \right) \frac{\alpha_t^4 \beta^2}{\sigma_t^4(\alpha_t^2 \beta +  \sigma_t^2)} + \landau{\frac1{\psi_d}} \eqsp.
\end{align*}

\textbf{Case 2.} ($\psi_p < \psi_d$) In that case, by very similar computations, we obtain
\begin{align*}
      \lim_{n \to \infty}\frac1{d} \inf_A  \left( \sigma_t^2 \LDSMn(s_{W,A}, t) \right) &= 1 - \frac{\psi_p}{\psi_d}  \left(1 - \frac1{\psi_p} \right) - \sigma_t^2 \frac{\psi_p}{\psi_d} \int \frac{f_{\mathrm{MP}}^{(\psi_p)}(\lambda)}{\alpha_t^2 \beta \lambda + \sigma_t^2} \der \lambda \\ 
      &= 1 + \frac1{\psi_d} - \frac{\psi_p}{\psi_d} + \landau[t]{\frac1{\psi_p \psi_d}} \eqsp.
\end{align*}
Finally, for the score matching loss, we obtain
\begin{align*}
     \lim_{n \to \infty} \left(\frac1{d} \LESM (s_{W,\widehat{A}}, t) \right) &= \frac{\psi_p}{\psi_d}\left( 1 - \frac1{\psi_p} \right) \frac{\alpha_t^4 \beta^2}{\sigma_t^4(\alpha_t^2 \beta +  \sigma_t^2)} + \frac{1 - \frac{\psi_p}{\psi_d}}{\alpha_t^2 \beta + \sigma_t^2}  \\& \qquad+ 
    (\alpha_t^2 \beta + \sigma_t^2) \frac{\psi_p}{\psi_d} \int \left( \frac1{\alpha_t^2 \beta \lambda + \sigma_t^2} - \frac1{\alpha_t^2 \beta + \sigma_t^2}  \right)^2  f_{\mathrm{MP}}^{(\psi_p)} (\lambda) \der \lambda \\
    &= \frac{\psi_p}{\psi_d}\left( 1 - \frac1{\psi_p} \right) \frac{\alpha_t^4 \beta^2}{\sigma_t^4(\alpha_t^2 \beta +  \sigma_t^2)}  + \frac{1 - \frac{\psi_p}{\psi_d}}{\alpha_t^2 \beta + \sigma_t^2}  + \landau{\frac1{\psi_d}} \eqsp.
\end{align*}

The result follows by rearranging the terms.
\end{proof}

\subsection{Omitted proofs of \Cref{sec:time-smoothness}}
\label{sec:proofs-time-smoothness}

In this section, we present the proofs of \Cref{sec:time-smoothness}, related to the impact of the time integration in the score matching loss.

\subsubsection{Proof of \Cref{lemma:time-smoothness-non-asymptitc-lemma}}

\begin{proof} (of \Cref{lemma:time-smoothness-non-asymptitc-lemma})
    By calculations similar to the proof of \Cref{lemma:fixed-time-non-asymptotic-lemma}, we can write the denoising score matching loss as
    \begin{align*}
        \LDSMn(s_{A,W}, \varpi) &= \sigma_\pi^{-2} \intOT \left( d + \sigma_t^2 \trace ( A^T A W \Sigmabfhat_t W^T ) + 2\sigma_t^2 \trace (AW) \right) \der \pi(t) \\
        &= \frac{d}{\sigma_\pi^2} + \trace ( W^T A^T A W \Sigmabfhat_\varpi  ) + 2 \trace (AW) \eqsp.
    \end{align*}
    By reasoning as in the proof \Cref{lemma:fixed-time-non-asymptotic-lemma}, we observe that (because $p>d$) all empirical risk minimizers $\Ahat$ satisfy
    \begin{align*}
        \Ahat W = -\Sigmabfhat_\varpi^{-1} \eqsp.
    \end{align*}
    Therefore, the minimum empirical risk is equal to
    \begin{align*}
         \inf_A \LDSMn(s_{A,W}, \varpi) = \frac{d}{\sigma_\pi^2} - \trace \left( \Sigmabfhat_\varpi^{-1} \right) \eqsp.
    \end{align*}
    Again, by similar calculations as in the proof of \Cref{lemma:fixed-time-non-asymptotic-lemma}, we can express the population score matching loss as
    \begin{align*}
        \LESM(s_{\Ahat,W}, \pi) = \intOT \Eof{\normof{\Ahat W (\alpha_t Z + \sigma_t \Xi) + \Sigmabf_t^{-1} (\alpha_t Z + \sigma_t \Xi)}^2} \der \pi(t) \eqsp.
    \end{align*}
    with $(Z,\Xi) \sim \nu \otimes \Nrm (0, I_d)$. 
    Therefore, we have 
    \begin{align*}
        \LESM(s_{\Ahat,W}, \pi) &= \intOT \trace \left( \Sigmabfhat_\varpi^{-2} \Sigmabf_t - 2 \Sigmabfhat_\varpi^{-1} + \Sigmabf^{-1}_t \right) \der \pi(t)\\
        &= \trace \left( \Sigmabfhat_\varpi^{-2} \Sigmabf_\pi - 2 \Sigmabfhat_\varpi^{-1} + \intOT \Sigmabf_t^{-1} \der \pi(t) \right)\\
        &= \frobnorm{\left( \Sigmabfhat_\varpi^{-1} - \Sigmabf_\pi^{-1} \right) \Sigmabf_\pi^{1 / 2} }^2 + \Crm_\pi \eqsp,
    \end{align*}
    with
    \begin{align}
        \label{eq:time-integrated-ESM_constant}
        \Crm_\pi := \trace \left(  \intOT \Sigmabf_t^{-1} \der \pi(t) - \Sigmabf_\pi^{-1} \right) \eqsp.
    \end{align}
    By diagonalizing $\Sigma_t$ and using Jensen's inequality, we see that $\Crm_\pi \geq 0$.
\end{proof}

\begin{remark}[Case where $p\leq d$.]
    \label{rk:case-small-p-time-smooth}
    By proceeding like in the proof of \Cref{lemma:fixed-time-non-asymptotic-lemma}, we have that
\begin{align*}
    \inf_A \LDSMn(s_{A,W}, \varpi) = \frac{d}{\sigma_\pi^2} - \trace \left( W^T (W\Sigmabfhat_\varpi W^T )^{-1} W \right) \eqsp.
\end{align*}
Similarly, for the explicit score matching loss, we have
\begin{align*}
    \LESM(s_{\Ahat,W}, \pi)\frobnorm{\left(W^T (W \widehat{\Sigmabf}_\varpi W^T)^{-1} W - \Sigmabf_\pi^{-1} \right) \Sigmabf_\pi^{1/2}}^2 + \Crm_\pi \eqsp,
\end{align*}
where $\Ahat = -(W \Sigmabfhat_\varpi W^T)^{-1}$ is the empirical risk minimizer. 
These formulas are used to obtain the exact asymptotics presented in \Cref{fig:fixed-time-linear-psi_p}, using very similar computations as in the proofs of \Cref{sec:fine-grainded-linear-analysis}.
\end{remark}

\begin{remark}[Case where $\pi$ is the uniform distribution and $\kappa=1$.]
\label{eq:explicit_crm_pi_computations}
In that case, we have
\begin{align*}
   \frac1{d} \Crm_\pi = \frac1{T} \int_0^T \frac{\der t}{\alpha_t^2 \beta + \sigma_t^2} - \frac1{\alpha_\pi^2 \beta + \sigma_\pi^2} \eqsp.
\end{align*}
We easily see that
\begin{align*}
    \alpha_\pi^2 = \frac1{T} \int_0^T e^{-2t} \der t = \frac{1 - e^{-2T}}{2T} \eqsp, \quad \sigma_\pi^2 = 1 - \frac{1 - e^{-2T}}{2T} \eqsp.
\end{align*}
Therefore, we have
\begin{align*}
    \frac1{d} \Crm_\pi &= \frac1{T} \int_0^T \frac{\der t}{e^{-2t} \beta + 1 - e^{-2t}} - \frac{2T}{\beta (1 - e^{-2T}) + 2T - 1 + e^{-2T}} \\
    &= \frac1{2T} \log \left( \frac{e^{2T} + \beta - 1}{\beta} \right) - \frac{2T}{\beta (1 - e^{-2T}) + 2T - 1 + e^{-2T}} \eqsp.
\end{align*}
Therefore, we have
\begin{align}
    \label{eq:time-integrated-constant-uniform-case}
     \frac1{d} \Crm_\pi = \frac1{2T}\log \left(\frac{ 1 + e^{-2T} (\beta - 1) }{\beta} \right) + \frac{\beta (1 - e^{-2T}) - 1 + e^{-2T}}{\beta (1 - e^{-2T}) + 2T - 1 + e^{-2T}} = \landau{\frac1{T}} \eqsp.
\end{align}
\end{remark}

\subsubsection{Proof of \Cref{prop:time-smoothness-asymptotics}}

We provide below the proof of \Cref{prop:time-smoothness-asymptotics}. We first provide, in the next proposition, the exact asymptotics of the score matching losses in the case of time-integrated losses.

\begin{proposition}
    \label{eq:time-smoothness-learning-curves}
    Consider the same setup as in \Cref{prop:time-smoothness-asymptotics}, we have
        \begin{align*}
        \lim_{n \to \infty} \left( \frac1{d} \inf_A \LDSMn(s_{A,W}, \varpi) \right) &= \frac1{\sigma_\pi^2} - \int \frac{\der \mu_{\mathrm{MP}}^{(\psi_d)} (\lambda)}{\alpha_\varpi^2 \beta \lambda + \sigma_\varpi^2} \eqsp,
    \end{align*}
    and
    \begin{align*}
       \lim_{n\to\infty} \left( \frac1{d}\LESM(s_{\Ahat, W}, \pi)\right) =  (\alpha_\pi^2 \beta + \sigma_\pi^2) \int \left( \frac1{\alpha_\varpi^2 \beta \lambda + \sigma_\varpi^2} - \frac1{\alpha_\pi^2 \beta + \sigma_\pi^2} \right)^2 \der \mu_{\mathrm{MP}}^{(\psi_d)} (\lambda) + \overline{\Crm}_\pi \eqsp,
    \end{align*}
    where, as before, $\mu_{\mathrm{MP}}^{(\psi_d)}$ is the Marchenko-Pastur distribution with shape parameter $\psi_d$, and 
    \begin{align}
        \label{eq:normalized_C_pi}
        \overline{\Crm}_\pi := \intOT \frac{\pi(\der t)}{\alpha_t^2 \beta + \sigma_t^2} - \frac1{\alpha_\pi^2 \beta + \sigma_\pi^2} \geq 0 \eqsp.
    \end{align}
\end{proposition}

\begin{proof} 
    Assume that $\psi_p > \psi_d$. Then, by \Cref{lemma:fixed-time-non-asymptotic-lemma} (using the same notations), we have
    \begin{align*}
        \frac1{d} \inf_A \LDSMn(s_{A,W}, \varpi)  &= \frac1{\sigma_\pi^2} - \int \frac{\der \mu_{\beta^{-1}\Sigmabfhat} (\lambda)}{\alpha_\varpi^2 \beta \lambda + \sigma_\varpi^2} \eqsp.
    \end{align*}
    By the Marchenko-Pastur's theorem, we have
    \begin{align*}
        \lim_{n \to \infty} \left( \frac1{d} \inf_A \LDSMn(s_{A,W}, \varpi) \right) &= \frac1{\sigma_\pi^2} - \int \frac{\der \mu_{\mathrm{MP}}^{(\psi_d)} (\lambda)}{\alpha_\varpi^2 \beta \lambda + \sigma_\varpi^2} \eqsp.
    \end{align*}
    Similarly, the score matching loss can be expressed as
    \begin{align*}
        \frac1{d}\LESM(s_{\Ahat, W}, \pi) = \frac{\alpha_\pi^2 \beta + \sigma_\pi^2}{d} \frobnorm{\Sigmabfhat_\varpi^{-1} - \frac{\Irm_d}{\alpha_\pi^2 \beta + \sigma_\pi^2}}^2 + \frac{\Crm_\pi}{d} \eqsp,
    \end{align*}
    where $\Crm_\pi$ is defined in the proof of \Cref{lemma:time-smoothness-non-asymptitc-lemma}. 
    Under \Cref{ass:linear_setup_assumption}, we observe that $\Crm_\pi / d$ is independent of $d$, and therefore we can define
    \begin{align*}
        \overline{\Crm}_\pi := \frac{\Crm_\pi}{d} =  \intOT \frac{\pi(\der t)}{\alpha_t^2 \beta + \sigma_t^2} - \frac1{\alpha_\pi^2 \beta + \sigma_\pi^2} \geq 0 \eqsp,
    \end{align*}
    where the inequality follows from Jensen's inequality. 
    By applying the Marchenko-Pastur theorem, we obtain
    \begin{align*}
        \lim_{n\to \infty} \left( \frac1{d}\LESM(s_{\Ahat, W}, \pi) \right) = (\alpha_\pi^2 \beta + \sigma_\pi^2) \int \left( \frac1{\alpha_\varpi^2 \beta \lambda + \sigma_\varpi^2} - \frac1{\alpha_\pi^2 \beta + \sigma_\pi^2} \right)^2 \der \mu_{\mathrm{MP}}^{(\psi_d)} (\lambda) +  \overline{\Crm}_\pi \eqsp.
    \end{align*}
    This concludes the proof.
\end{proof}

\begin{remark}
    Based on \Cref{rk:case-small-p-time-smooth}, we can easily extend these derivations to the case $\psi_p < \psi_d$.
\end{remark}

We present below the proof of \Cref{prop:time-smoothness-asymptotics}.
\begin{proof} (of \Cref{prop:time-smoothness-asymptotics})
    Assume that $\psi_p > \psi_d > 1$, then the Marchenko-Pastur has a mass at zero in all the cases below. By \Cref{eq:time-smoothness-learning-curves}, we immediately have
     \begin{align*}
        \lim_{n \to \infty} \left( \frac1{d} \inf_A \LDSMn(s_{A,W}, \varpi) \right) &= \frac1{\sigma_\pi^2} - \frac1{\sigma_\varpi^2} \left( 1 - \frac1{\psi_d} \right) - \int \frac{f_{\mathrm{MP}}^{(\psi_d)} (\lambda)}{\alpha_\varpi^2 \beta \lambda + \sigma_\varpi^2}  \der \lambda \\
        &= \frac1{\sigma_\pi^2} - \frac1{\sigma_\varpi^2}  + \frac1{\psi_d \sigma_\varpi^2}  + \landau{\frac1{\psi_d^2}} \eqsp.
    \end{align*}
    Note that, by the Cauchy-Schwarz inequality, we always have $\sigma_\pi^{-2} \geq \sigma_\varpi^{-2}$. 
    Similarly, the asymptotic score matching loss satisfies
    \begin{align*}
         \lim_{n\to \infty} \left( \frac1{d}\LESM(s_{\Ahat, W}, \pi) \right) = \left( 1 - \frac1{\psi_d} \right) \frac{(\alpha_\pi^2 \beta + \sigma_\pi^2 - \sigma_\varpi^2)^2}{\sigma_\varpi^4 (\alpha_\pi^2 \beta + \sigma_\pi^2)} + \overline{\Crm}_\pi + \landau{\frac1{\psi_d}} \eqsp.
    \end{align*}
    The proof follows by defining
    \begin{align}
        \label{eq:Lrm_pi_constant}
        \Lrm_\pi := \frac{(\alpha_\pi^2 \beta + \sigma_\pi^2 - \sigma_\varpi^2)^2}{\sigma_\varpi^4 (\alpha_\pi^2 \beta + \sigma_\pi^2)} +  \overline{\Crm}_\pi
    \end{align}
    This concludes the proof.
\end{proof}

\subsection{Proof of \Cref{sec:early-stopping}}
\label{sec:proof-early-stopping}

We present below the proof of \Cref{prop:early-stopping-proposition}.

\begin{proof}(of \Cref{prop:early-stopping-proposition})
    \textbf{Step 1.} We start by solving the gradient flow dynamics of \Cref{eq:gradient-flow-dynamics}.

    We can expand the gradient flow dynamics as
\begin{align*}
    \frac{\der}{\der \tau}A(\tau) = -2 p^{-1} \sigma_t^2 A(\tau) W \Sigmabfhat_t W^T - 2 p^{-1} \sigma_t^2 W^T \eqsp.
\end{align*}
We can assume, without loss of generality, that $p > d$. Then, by reasoning as in the proof of \Cref{prop:fixed-time-asymptotic-linear-results}, we observe that the matrix $W^T W \in \R^{d \times d}$ is almost-surely invertible. Therefore, we can define $\Tilde{A} := (W^T W)^{-1} A W$ and we have almost-surely that
\begin{align*}
   \frac{\der}{\der \tau} \Tilde{A}(\tau) = -2 p^{-1}  \sigma_t^2 \Tilde{A}(\tau) \Sigmabfhat_t W^T W - 2 p^{-1}  \sigma_t^2 \Irm_d \eqsp. 
\end{align*}
Taking into account that the initialization satisfies $A(0) = 0$, we deduce that
\begin{align*}
    \Tilde{A}(\tau) = - (\Sigmabfhat_t W^T W)^{-1} \left( \Irm_d - e^{-2\tau p^{-1} \sigma_t^2 \Sigmabfhat_t W^T W} \right) \eqsp.
\end{align*}
Therefore, we have
\begin{align*}
    A(\tau) W = - \Sigmabfhat_t^{-1} \left( \Irm_d - e^{-2\tau p^{-1} \sigma_t^2 \Sigmabfhat_t W^T W} \right) \in \R^{d \times d} \eqsp.
\end{align*}
By the strong law of large numbers, we have, almost-surely (for fixed $n$, $d$, and $\tau$), that $p^{-1} W^T W \to \Irm_d$ as $p\to\infty$. Thus, almost-surely, we have
\begin{align}
    \label{eq:gradient-flow-infinite-width-limit}
    \lim_{p\to \infty} (A(\tau)W) = - \Sigmabfhat_t^{-1} \left( \Irm_d - e^{-2\tau \sigma_t^2 \Sigmabfhat_t } \right) \in \R^{d \times d} \eqsp.
\end{align}

    \textbf{Step 2.} By computations similar to the proof of \Cref{lemma:fixed-time-non-asymptotic-lemma} and \Cref{prop:fixed-time-asymptotic-linear-results}, we have (for fixed $(p,d,n)$)
    \begin{align*}
      \ecal_t(\tau, n, d, p) :=  \LESM(s_{W, A(\tau)}, t) = \frobnorm{\left( \Sigmabfhat_t^{-1}  \left( \Irm_d - e^{-2\tau p^{-1} \sigma_t^2 \Sigmabfhat_t W^T W} \right) - \Sigmabf_t^{-1} \right) \Sigmabf_t^{1 / 2}}^2 \eqsp.
    \end{align*}
    Therefore, by the law of large numbers, we have almost surely that
    \begin{align*}
        \lim_{p\to \infty} \ecal(\tau,n,d, p) &=  \frobnorm{\left( \Sigmabfhat_t^{-1}  \left( \Irm_d - e^{-2\tau  \sigma_t^2 \Sigmabfhat_t } \right) - \Sigmabf_t^{-1} \right) \Sigmabf_t^{1 / 2}}^2 \\
        &= (\alpha_t^2 \beta + \sigma_t^2) \frobnorm{ \Sigmabfhat_t^{-1}  \left( \Irm_d - e^{-2\tau  \sigma_t^2 \Sigmabfhat_t } \right) - \Sigmabf_t^{-1}}^2 =: \overline{\ecal}_t(\tau, n, d) \eqsp.
    \end{align*}
    Recall that we take $n,d \to \infty$ with $\psi_d := \lim(d / n) > 1$.
    Therefore, we can apply the Marchenko-Pastur theorem and obtain, through similar computations as in the proof of \Cref{prop:fixed-time-asymptotic-linear-results}, that
    \begin{align*}
    \lim_{n\to\infty}\frac1{d}\overline{\ecal}_t(\tau, n, d) &= (\alpha_t^2 \beta + \sigma_t^2) \int \left( \frac{1 - e^{-2\tau \sigma_t^2 (\alpha_t^2 \beta \lambda + \sigma_t^2)}}{\alpha_t^2 \beta \lambda + \sigma_t^2} - \frac1{\alpha_t^2 \beta + \sigma_t^2}  \right)^2 \der \mu_{\mathrm{MP}}^{(\psi_d)}(\lambda)\\
        &= \left( 1 - \frac1{\psi_d} \right) \frac{\left((\alpha_t^2 \beta + \sigma_t^2) \left( 1 - e^{-2\tau \sigma_t^4 } \right) - \sigma_t^2 \right)^2}{\sigma_t^4 (\alpha_t^2 \beta + \sigma_t^2)} + \landau{\frac1{\psi_d}} \eqsp.
    \end{align*}
    This concludes the proof.
\end{proof}

\section{Experimental details}
\label{sec:experimental-details}

In this section, we provide some additional details on the experiments presented in the main text.

\subsection{Random features experiments}
\label{sec:random-feature-experiment-details}

In this short subsection, we quickly describe the empirical setup for the diffusion random feature experiment presented in \Cref{fig:intro_plot_random_features}.

\paragraph{Regression experiments.}
In the regression experiment presented in \Cref{fig:intro_plot_random_features}, we use a classical random feature regression setting. We consider a data distribution of the form
\begin{align*}
    y_i = \langle \upbeta_\star, x_i \rangle + \sigma \epsilon_i \eqsp, \quad x_i \sim \Nrm(0, \Irm_d)\eqsp, \quad \epsilon \sim \Nrm(0, 1) \eqsp, 
\end{align*}
where $\epsilon_i$ and $x_i$ are independent and we sample $n$ i.i.d. pairs $(x_i,y_i)$ for $i \in \setof{1,\dots,n}$.
We consider a model of the form
\begin{align*}
    f_{A,W}(x) = \frac{A}{\sqrt{p}} \varrho \left( \frac{W}{\sqrt{d}} x \right)  \eqsp,
\end{align*}
where $A \in \R^{d \times p}$, $W \in \R^{p \times d}$ is a random matrix with i.i.d. standard Gaussian entries, and $\varrho (x) := \max(x,0)$ is the ReLU activation.

We consider the minimization of the empirical risk
\begin{align*}
    \Ahat \in \argmin_A \left\{ \frac1{n} \sum_{i=1}^n (y_i - f_{A,W}(x_i))^2 + \frac{\lambda}{\sqrt{p}} \frobnorm{A}^2  \right\} \eqsp,
\end{align*}
with $\lambda$ a ridge regularization parameter.

Then, we evaluate the test (using a validation set) and train risk evaluated at the empirical risk minimizer $\Ahat$ and report it in \Cref{fig:intro_plot_random_features}.

\textbf{Hyperparameters details.} In \Cref{fig:intro_plot_random_features}, we use $n=10^2$, $d=20$, $N_g = 10$, $\lambda = 10^{-4}$, and $p \in [5, 4 \cdot 10^2]$.

\paragraph{Diffusion experiments.} 
We consider the following family of random features score networks, which is similar to existing works \cite{bonnaire_why_2025-1,george_asymptotic_2026}, for $x \in \Rd$
\begin{align*}
    f_{A,W}(x) = \frac{A}{\sqrt{p}} \varrho \left( \frac{W}{\sqrt{d}} x \right) \eqsp,
\end{align*}
where $A \in \R^{d \times p}$, $W \in \R^{p \times d}$ is a random matrix with i.i.d. standard Gaussian entries, and $\varrho (x) := \max(x,0)$ is the ReLU activation. For the data distribution, we take the data distribution to by $\nu = \Nrm(0, d^{-1}\Irm_d)$.

Given a fixed $t > 0$, a random $W$ and a dataset $S^{(n)} := (Z_1,\dots,Z_n) \sim \nu^{\otimes n}$, we consider the minimization of the regularized empirical risk, defined by
\begin{align*}
    A \mapsto \frac1{n} \sum_{i=1}^n \Eof[\Xi]{\normof{\sigma_t f_{A,W} (\alpha_t Z_i + \sigma_t \Xi) + \Xi}^2} + \frac{\lambda}{\sqrt{p}} \frobnorm{A}^2 \eqsp,
\end{align*}
with $\Xi \sim \Nrm(0,\Irm_d)$ and $\lambda$ a ridge regularization parameter. An important difference with the regression case above is that, in order to make the empirical risk minimization tractable, we use a Monte Carlo approximation of the expectation over $\Xi \sim \Nrm(0,\Irm_d)$. Given an integer $N_g \in \N^\star$ and $(\xi_{ij})_{1 \leq i \leq n,~1 \leq j \leq N_g} \sim \Nrm(0,\Irm_d)^{\otimes (n N_g)}$, we therefore compute
\begin{align*}
    \Ahat \in \argmin_A \left\{ \frac1{n} \sum_{i=1}^n \sum_{j=1}^{N_g} {\normof{\sigma_t f_{A,W} (\alpha_t Z_i + \sigma_t \xi_{ij}) + \xi_{ij}}^2} + \frac{\lambda}{\sqrt{p}} \frobnorm{A}^2  \right\}
\end{align*}
by noting that it can be written as the solution of a linear regression problem. This procedure is similar to \cite{george_denoising_2025}.

After estimating $A$, we compute the empirical and population denoising score matching loss, where the latter is estimated using a validation set. The hyperparameters details are provided below.

\textbf{Hyperparameters details.} In \Cref{fig:intro_plot_random_features}, we use $\kappa = 1$, $n=2 \cdot 10^4$, $t = 10^{-1}$, $d=3$, $N_g = 10$, $\lambda = 10^{-4}$, and $p \in [2, 10^4]$.

\subsection{High-dimensional experiments}
\label{sec:unet_experiments}

Our experiments are performed using an implementation of the DDPM model from \cite{song2021scorebased}. We use the configuration titled\texttt{vp.ddpm.cifar10\_continuous} which implements the DDPM model of \cite{hoDenoisingDiffusion2020} but for the continuous-time variance preserving setting. The architecture is a U-Net \citep{ronnenberger_unet_2015} with the encoder and decoder each consisting of four resolution levels $(32\times32, 16\times16, 8\times8, 4\times4)$, with two residual blocks per level and utilizes self-attention at a resolution of $16\times16$. The model is conditioned on time, with the timestep encoded via a sinusoidal embedding and injected into each residual block.

For figures \ref{fig:intro_figure_experiment-unet} and \ref{fig:train_unet} we modify the number of features (NF) parameter (denoted \texttt{nf} in the configuration file and taking a default value of $128$). This parameter defines the base channel width used throughout the model which is multiplied by $m = 1, 2, 2, 2$ to produce the actual channel widths for each of the four resolution levels. The number of parameters in the model scales quadratically with channel width. We consider NF$\in \{2, 4, 8, 16, 32, 64, 128\}$ to see how the number of parameters affects overparameterization. Some technical modifications to the code were required to support smaller values of NF.

For all high-dimensional experiments we use CIFAR-10 but restrict to the first $1,000$ examples to make overfitting easier to achieve. Evaluation is performed on the full $10,000$ held-out examples. We train the model with the default setup in the codebase (ADAM with $10\%$ dropout and EMA decay $0.9999$) and we train for $1,000,000$ iterations, which is larger than typical to better guarantee overfitting. In the experiments of Figure \ref{fig:intro_figure_experiment-unet} we choose the checkpoint with smallest train loss. The samples generated on the right-hand side of Figure \ref{fig:intro_figure_experiment-unet} uses the DDIM implementation in the codebase to generate samples with a fixed seed across NF values. We perform log-likelihood calculations using the \texttt{bpd} implementation in the codebase.

For the experiment in Figure \ref{fig:no_time_exp}, we consider the DDPM network but we restrict the access of the neural network to the time input. However, we do still scale the network properly according to the time input. We then train the model as usual but restrict the time variable to the range $[t, t + r]$ for $t=0.1$ and $r \in [0.0001, 0.00032, 0.001, 0.0032, 0.01, 0.032, 0.1, 0.5]$. We then evaluate the train and test error at $t$ to obtain the plot.

\subsection{Compute resources}
\label{sec:compute-resources}

Our experiments were run on an Amazon EC2 G6e (g6e.xlarge) instance which has a single NVIDIA L40S Tensor Core GPU.

\subsection{Licenses}
\label{appendix:licenses}
Codebases:
\begin{itemize}
    \item Score-Based Generative Modeling through Stochastic Differential Equations \citep{song2021scorebased}: Apache License 2.0.
\end{itemize}

Datasets:
\begin{itemize}
    \item CIFAR-10 \citep{krizhevsky2009learning}: MIT license.
\end{itemize}

\end{document}